\documentclass[11pt]{article}
\usepackage{style}
\usepackage[left=1in,bottom=1in,top=1in,right=1in]{geometry}
\hypersetup{%
	colorlinks   = true, %
	urlcolor     = blue, %
	linkcolor    = blue, %
	citecolor    = blue,  %
}

\title{From Non-Convex to Strongly Convex:\\
Curvature-Adaptive FTPL for Online Optimization}
\author{Moses Charikar\thanks{Stanford University. Email: \texttt{moses@cs.stanford.edu.}}
\and
Chirag Pabbaraju\thanks{Stanford University. Email: \texttt{cpabbara@cs.stanford.edu.}}
\and
Ambuj Tewari\thanks{University of Michigan, Ann Arbor. Email: \texttt{tewaria@umich.edu.}}
}
\date{\today}

\graphicspath{{img/}}

\begin{document}

\maketitle

\begin{abstract}
Curvature adaptivity is a classical theme in online optimization: for convex Lipschitz losses, adaptive methods interpolate between the optimal $O(\sqrt{T})$ regret for general convex losses and $O(\log T)$ regret under strong convexity. Recent work has shown that Follow-the-Perturbed-Leader (FTPL) achieves optimal $O(\sqrt{T})$ regret even for online non-convex Lipschitz losses, assuming access to an approximate offline-optimization oracle, but these guarantees do not exploit curvature. We show that FTPL can be made curvature-adaptive in the non-convex setting, without knowing in advance how curvature will accumulate over time. Our algorithm replaces the fixed perturbation scale of standard FTPL with a time-varying scale chosen using only past information. We give a simple follow-the-leader tuning rule for this scale and show that it competes, up to constants, with the best choice in hindsight. The resulting method achieves $O(\sqrt{T})$ regret for arbitrary non-convex Lipschitz losses and improves as cumulative curvature grows; with sufficiently accurate oracle calls, it achieves $O(\log T)$ regret when cumulative curvature grows linearly, which includes the classical strongly convex regime. We complement these upper bounds with matching lower bounds for prescribed cumulative-curvature sequences, already for one-dimensional convex losses, showing that the tradeoff between worst-case non-convex regret and curvature-driven fast rates is intrinsic.
\end{abstract}

\section{Introduction}

A central question in online optimization is how regret scales with the number of rounds $T$. Two classical results characterize this dependence: for convex and Lipschitz losses, the optimal regret is $O(\sqrt{T})$~\citep{zinkevich2003online}, while for strongly convex and Lipschitz losses, it improves to $O(\log T)$~\citep{hazan2006logarithmic}. Both rates are known to be tight~\citep{abernethy2008optimal}. 

A natural goal, therefore, is to design algorithms whose guarantees adapt to the curvature of the loss sequence. In a seminal work,~\citet{bartlett2007adaptive} introduced adaptive gradient methods for online \emph{convex} optimization whose regret seamlessly interpolates between $O(\sqrt{T})$ and $O(\log T)$ depending on the cumulative strong convexity in the sequence. These ideas have since influenced widely used methods such as Adagrad~\citep{duchi2011adaptive} and Adam~\citep{kingma2015adam}. More recently,~\citet{suggala2020online} showed that the Follow-the-Perturbed-Leader (FTPL) framework achieves $O(\sqrt{T})$ regret even for \emph{non-convex} Lipschitz losses, assuming access to an oracle that solves perturbed offline problems.

These results naturally raise the following question:
\begin{quote}
\emph{Can FTPL be made adaptive to curvature, achieving $O(\sqrt{T})$ regret for general non-convex Lipschitz losses while improving to $O(\log T)$ when sufficient strong convexity is present?}
\end{quote}

In this paper, we answer this question in the affirmative. Our approach is based on introducing a time-varying noise scale in the FTPL perturbations, chosen adaptively using only past information. We first derive a general regret bound parameterized by this sequence of noise scales. We then show that selecting these parameters can itself be cast as a meta online learning problem, for which a simple FTL approach achieves a constant competitive ratio. The resulting algorithm achieves regret that interpolates between $O(\sqrt{T})$ and $O(\log T)$, without prior knowledge of the curvature of the losses. We further complement our upper bounds with matching lower bounds, showing that the bounds we meta-optimize are intrinsic to the problem and not an artifact of our method.

Beyond its theoretical interest, our setting is motivated by a broad class of problems where the learner encounters composite losses of the form
\[
f_t = g_t + r_t,
\]
where $g_t$ is a data-dependent term that may be non-convex, and $r_t$ is a strongly convex regularizer. Such composite structures are ubiquitous in modern machine learning and optimization.

One prominent example arises in continual learning~\citep{kirkpatrick2017overcoming}, where $g_t$ represents a task-specific loss, typically non-convex due to neural network parameterizations, and $r_t$ is a regularizer that \emph{tethers the current iterate to previously learned models}, thereby mitigating catastrophic forgetting. In contrast, in online robust regression~\citep{maronna2019robust}, even with a linear model, the loss $g_t$ can be non-convex when using robust losses such as the Tukey biweight function, while $r_t$ is typically a strongly convex shrinkage term (e.g., $\ell_2$ regularization) that \emph{tethers the iterate to a fixed reference point}, such as the origin.

In these settings, the effective curvature of the loss sequence depends on the interplay between the non-convex data term and the accumulated regularization. In favorable regimes, the regularizer dominates and the problem behaves like a strongly convex one; in adversarial regimes, the non-convexity prevails. Since this balance is typically unknown a priori, it is essential to design algorithms that automatically adapt, achieving $O(\sqrt{T})$ regret in the worst case while improving to $O(\log T)$ whenever sufficient curvature emerges.

Our key contributions are as follows.
\begin{itemize}
\item \textbf{Curvature-adaptive FTPL across non-convex to strongly convex regimes.}
We design an FTPL algorithm with a time-varying perturbation scale that adapts to the curvature of the loss sequence, achieving regret that interpolates from $O(\sqrt{T})$ for arbitrary non-convex Lipschitz losses to $O(\log T)$ when sufficient strong convexity is present. To our knowledge, this is the first FTPL method that achieves such curvature adaptivity in the non-convex regime, bridging the gap between prior adaptive methods for online convex optimization (e.g.,~\citep{bartlett2007adaptive}) and non-convex but non-adaptive FTPL guarantees (e.g.,~\citep{suggala2020online}).

\item \textbf{Meta-online tuning and optimality guarantees.}
We derive a regret bound parameterized by the perturbation schedule and show that selecting this schedule can be cast as a meta online learning problem. We give a simple constant-competitive strategy for this problem, yielding near-optimal tuning without prior knowledge of the loss sequence. We also prove matching lower bounds, showing that the bounds we meta-optimize are intrinsic to the problem.

\item \textbf{Empirical validation.}
We empirically evaluate our method on representative problems with composite non-convex and strongly convex structure, demonstrating improved performance over non-adaptive FTPL and baseline methods in regimes with varying curvature.
\end{itemize}
\section{Related Works}

\paragraph{Oracles in Online Optimization.} The FTPL algorithm~\citep{hannan1957approximation,kalai2005efficient} is especially attractive for efficient online learning over structured decision sets. The main computational burden in FTPL at each round is solving a perturbed offline problem. If the cost functions and perturbation are both linear, then the perturbed offline problem becomes an instance of the offline optimization problem itself. However, for general decision sets, there is known to be an exponential (in the ambient dimension $d$) gap in the oracle complexities of online and statistical learning if the learner is only provided access to the offline optimization oracle~\citep{hazan2016computational}. 
For decision classes with exponentially many experts in the ambient dimension, the lower bound of \citet{hazan2016computational} implies that no algorithm using only a standard offline optimization oracle and polynomial additional computation can achieve $O(\mathrm{poly}(d)T^\alpha)$ regret for any $\alpha<1$.
It was therefore surprising when \citet{agarwal2019learning} showed that FTPL with linear perturbations, which makes one call per time step to a \emph{linearly perturbed} offline optimization oracle, achieves $O(\mathrm{poly}(d)T^{2/3})$ cumulative regret against \emph{non-convex} Lipschitz functions on a  bounded decision set. The $O(\mathrm{poly}(d)T^{2/3})$ rate was later improved to $O(\mathrm{poly}(d)\sqrt{T})$ by \citet{suggala2020online}. 

\paragraph{Adaptive Step-Sizes and Regularization.} Online gradient algorithms with different step-size schedules achieve $O(\sqrt{T})$ and $O(\log T)$ regret against convex Lipschitz and strongly convex Lipschitz functions respectively~\citep{zinkevich2003online,hazan2006logarithmic}. \citet{bartlett2007adaptive} designed a single adaptive algorithm that interpolates between these rates and automatically adapts to the actual amount of curvature present in the adversary's \emph{convex} functions. Some related lower bounds can be found in the work of~\citet{abernethy2008optimal}. Subsequent work developed a broader theory of adaptive regularization and adaptive first-order methods, including AdaGrad-style algorithms~\citep{duchi2011adaptive,mcmahan2010adaptive} and adaptive FTRL/mirror-descent frameworks~\citep{mcmahan2017survey}. These ideas later influenced practical optimization methods for deep learning such as Adam~\citep{kingma2015adam}. While adaptive step-size methods are also widely studied for stochastic non-convex optimization~\citep{li2019convergence,ward2020adagrad}, their guarantees are typically formulated in terms of convergence to stationary points rather than regret minimization. Also note that, as observed by \citet{agarwal2019learning}, standard FTRL-style regularization does not necessarily stabilize the minimizers of non-convex objectives, limiting the direct applicability of FTRL-based approaches for provable regret guarantees in the non-convex regime.

\paragraph{Adaptive Perturbation.}
Adaptive variants of FTPL have previously been studied through adaptive learning-rate schedules, data-dependent perturbation magnitudes, and generalized FTPL frameworks for oracle-efficient online learning~\citep{hutter2005adaptive,wang2022adaptive}. Existing works primarily focus on adapting to favorable data sequences, for example when the loss of the best expert is small or the adversary is stochastic, while retaining the classical $O(\sqrt{T})$ dependence on the horizon. In contrast, our goal is to adapt to curvature itself. To the best of our knowledge, no prior online optimization algorithm, let alone one that is perturbation-based for oracle-efficiency,  achieves a full interpolation between the optimal $O(\sqrt{T})$ regret for Lipschitz non-convex losses and the $O(\log T)$ regret achievable for strongly convex losses.
\section{Adaptive FTPL}
\label{sec:adaptive-ftpl}

\paragraph{Problem Setup.} The setup comprises of an online game between an \textit{adversary} and a \textit{learner} that proceeds in rounds. At round $t$, the learner chooses some $x_t \in \mcX$, where $\mcX$ is a bounded space; simultaneously, the adversary picks a function $f_t:\mcX \to \R$. Both the learner's and adversary's choices at round $t$ can depend on the entire history of the game up until round $t-1$. Thereafter, the choices made by the learner and adversary are revealed to each other, and the learner suffers a loss $f_t(x_t)$. The learner's objective is to minimize their \textit{regret} at the end of the game; that is, the learner aims to minimize the cumulative loss they suffer at the end of $T$ rounds, relative to the single best $x \in \mcX$ they could have played at each round in hindsight. Formally, define the regret of a sequence of moves $x_{1:T} := x_1,\dots,x_T$ with respect to a sequence of functions $f_{1:T} := f_1,\dots,f_T$ as
\begin{align}
    \label{eqn:regret-def}
    \regret_T(f_{1:T}; x_{1:T}) := \sum_{t=1}^T f_t(x_t) - \inf_{x \in \mcX} \sum_{t=1}^T f_t(x).
\end{align}
We will assume that $\mcX \subseteq \R^d$ with $\ell_\infty$ diameter $D < \infty$. We will also assume that the functions $f_t$ that the adversary can play are guaranteed to be $L$-Lipschitz with respect to the $\ell_1$ norm. That is, for every $x,y \in \mcX$, and for every $t$, $|f_t(x)-f_t(y)| \le L\|x-y\|_1$. 

\paragraph{Curvature.} We say that a function $f: \mcX \to \R$ has strong convexity constant $\mu > 0$ (with respect to $\ell_1$) if for every $x, y \in \mcX$, it holds that
\begin{align}
    \label{eqn:strong-convexity}
    f(y) \ge f(x) + \langle\nabla f(x), y-x\rangle + \frac{\mu}{2}\|y-x\|_1^2.
\end{align}
We use the convention that $\mu=0$ if \eqref{eqn:strong-convexity} above does not hold for $f$ for any $\mu > 0$.

\paragraph{Approximate Optimization Oracle.} As in \citet{suggala2020online}, we assume that the learner has access to an ``$(\alpha,\beta)$-approximate optimization oracle''. Given a function $f:\mcX \to \R$ and a vector $\sigma \in \R^d$, the oracle $\mcO_{\alpha, \beta}(f, \sigma)$ returns $\hat{x} \in \mcX$ satisfying
\begin{align}
    \label{eqn:approximate-optimization-oracle}
    f(\hat{x}) - \langle \sigma, \hat{x} \rangle \le \inf_{x \in \mcX} \left(f(x)-\langle \sigma, x \rangle \right) + (\alpha + \beta\|\sigma\|_1).
\end{align}

We can now state our main Adaptive FTPL algorithm, which takes as input the oracle $\mcO_{\alpha, \beta}$, and a non-decreasing schedule $\rho_1,\rho_2,\dots$ of \textit{noise scales} to invoke the oracle with.

\begin{algorithm}[H]
\caption{Adaptive Follow the Perturbed Leader (AdaFTPL)}
\label{alg:ftpl}
    \begin{algorithmic}[1]
      \State \textbf{Input:} Approximate optimization oracle $\mathcal{O}_{\alpha, \beta}$, non-decreasing noise scale schedule $\rho_1,\rho_2\dots,$
      \For{$t = 1 \dots T$}
      \State Generate random vector 
      $\sigma_t$ such that $\{\sigma_{t,j}\}_{j = 1}^d \stackrel{\mathrm{i.i.d.}}{\sim} \text{Exp}(1)\footnotemark$
      \State Let $F_{t-1}=\sum_{i=1}^{t-1}f_i$\, ; predict $x_t$ as
      \[
      x_t = \mcO_{\alpha, \beta} \left(F_{t-1},\, \rho_t \sigma_t \right)
      \]
      \State Observe loss function $f_t$
      \EndFor
    \end{algorithmic}
\end{algorithm}
\footnotetext{A random variable $X \sim \mathrm{Exp}(\eta)$ if for every $s \ge 0$, $\Pr[X \ge s]=\exp(-\eta s)$.}

We are now ready to state our main regret guarantee for Adaptive FTPL. We note that throughout the paper, we will use the convention $1/0=+\infty$.

\begin{restatable}[Adaptive FTPL Regret]{theorem}{mainTheorem}
    \label{thm:main}
    Let $D$ be the $\ell_\infty$ diameter of $\mcX$, and suppose that the functions $f_1,\dots,f_T$ encountered by Adaptive FTPL (\Cref{alg:ftpl}) are $L$-Lipschitz with respect to the $\ell_1$ norm. Furthermore, let $S_t$ be the strong convexity constant %
    of $F_{t} = \sum_{i=1}^{t}f_i$. Then, the algorithm has the following regret guarantee:
    \begin{align}
        \label{eqn:main-regret-bound}
         \E_\sigma\left[\regret_T(f_{1:T};x_{1:T})\right]\le &\sum_{t=1}^T \min \left\{\frac{125L^2d^2D}{\rho_t} + \frac{\alpha+\beta \rho_t d}{20} + 2\beta d L, \frac{2L^2}{S_t}+2L\sqrt{\frac{2(\alpha + \beta\rho_t d)}{S_t}}\right\} \nonumber \\
         &\qquad + d D\rho_T+ \alpha T +  d\beta \sum_{t=1}^T \rho_t.
    \end{align}
\end{restatable}
We give the complete proof of \Cref{thm:main} in \Cref{sec:proofs-adaptive-ftl}, but provide a detailed sketch here.

\paragraph{Proof Sketch.} By a standard argument for FTPL (e.g., see Lemma 12 in \citet{hutter2005adaptive}, Chapter 4 in \citet{Cesa-Bianchi_Lugosi_2006}), it suffices to consider the setting where the functions $f_1,\dots,f_T$ are specified upfront by an \textit{oblivious} adversary, and show the regret bound for every fixed $f_1,\dots,f_T$. With obliviously specified functions, it furthermore suffices in \Cref{alg:ftpl} to simply draw a random vector $\sigma=\{\sigma_j\}_{j=1}^d \stackrel{\mathrm{i.i.d.}}{\sim} \text{Exp}(1)$ once at the start and reuse it at every time step, instead of drawing a fresh $\sigma_t$ at each step. With these considerations, our goal is to upper bound the regret in the left-hand side of \eqref{eqn:main-regret-bound} for fixed $f_1,\dots,f_T$, where the expectation is only with respect to $\sigma$ drawn at the start by the algorithm.

The key object that enables the analysis is a ``ghost'' move $\bar{x}_{t+1}= \mcO_{\alpha, \beta}(F_{t}, \rho_t \sigma)$. Note that $\bar{x}_{t+1}$ is separate from $x_{t}=\mcO_{\alpha, \beta}(F_{t-1},\, \rho_t\sigma)$ which the algorithm plays at time $t$, \textit{as well as} $x_{t+1}=\mcO_{\alpha, \beta}(F_t,\, \rho_{t+1}\sigma)$ which the algorithm plays at time step $t+1$. We introduce $\bar{x}_{t+1}$ solely for the sake of analysis, and for the reason that it uses a noise vector whose scale \textit{matches} that of the noise vector used in the \textit{previous} time step. This will help us (1) establish a standard decomposition of the regret into a \textit{stability} and \textit{perturbation} term, and (2) directly utilize the bound on the stability term from \citet{suggala2020online}.

Namely, we show that
\begin{align}
    \label{eqn:stability-perturbation-decomposition-sketch}
    \E_\sigma\left[\regret_T(f_{1:T};x_{1:T})\right] \le \underbrace{L \sum_{t=1}^T \E_\sigma \left[ \left\| x_t - \bar{x}_{t+1}\right\|_1 \right]}_{\text{stability}} + \underbrace{d D\rho_T+ \alpha T +  d\beta \sum_{t=1}^T \rho_t.}_{\text{perturbation}}
\end{align}
This decomposition follows a standard analysis once the ghost object $\bar{x}_{t+1}$ has been defined. The stability term arises directly from Lipschitzness. In order to bound the rest of the perturbation terms, we crucially observe that $\bar{x}_{t+1}$ is a move that the learner would have \textit{liked to play} at time step $t$, but cannot do so, since it does not know $f_t$. Thus, we can employ a standard ``be-the-leader'' analysis for $\bar{x}_{t+1}$ to bound the perturbation terms in the desired form.

It then remains to bound the stability terms $\E_{\sigma}\left[\|x_t - \bar{x}_{t+1}\|_1\right]$. We do this in two ways, allowing us to interpolate between non-convexity and strong convexity. First, we observe again that both $x_t$ and $\bar{x}_{t+1}$ use a noise vector having the same scale $\rho_t$; the only difference is that $x_t$ invokes the oracle with $F_{t-1}$, whereas $\bar{x}_{t+1}$ invokes the oracle with $F_t$. But this means that we can directly use the analysis of \citet{suggala2020online} --- in their setting, they upper-bound $\E_{\sigma}\left[\|x_t - x_{t+1}\|_1\right]$ where both $x_t$ and $x_{t+1}$ use the same noise vector; the same upper bound goes through for us, provided we replace $x_{t+1}$ with $\bar{x}_{t+1}$.

Second, we use a different upper bound on $\|x_t - \bar{x}_{t+1}\|$ arising from the analysis of follow-the-leader (FTL) for strongly convex online optimization. Namely, if we assume that the strong convexity constant of the cumulative function $F_t$ is $S_t$, then $\|x_t - \bar{x}_{t+1}\|$ can be upper-bounded as $O(L/S_t)$. This gives a different, curvature-cognizant upper bound on the perturbation term. This bound is vacuous when $F_t$ is not strongly convex (given that we are using the convention $1/0=+\infty$), but it nevertheless always applies. 

Finally, given that we have two upper bounds for the stability term, both of which always hold, we can use the \textit{minimum} of these two. This gives us the final regret upper bound in \Cref{thm:main}. \qed

\subsection{A meta-online learning problem to choose $\rho_t$}

Observe that the regret bound in \Cref{thm:main} holds for any non-decreasing $\rho_1,\dots,\rho_T$. We now consider the problem of optimally choosing this sequence. A natural way to do this is to optimize the upper bound that we obtain, as a function of $\rho_1,\dots,\rho_T$.

Namely, for any horizon $T$, and error parameters $\alpha, \beta$, define the \textit{offline} optimal value of the regret upper bound over all possible non-decreasing noise scale schedules $\rho_1 \le \dots \le \rho_T$ that are lower-bounded by $r \ge 0$ (i.e., $r \le \rho_1$), as:
\begin{align}
    \label{eqn:offline-opt-rho}
    \opt^{\alpha,\beta}_T(r) := \inf_{r \le \rho_1 \le \dots \le \rho_T} &\left[\sum_{t=1}^T \min\left\{\frac{125 L^2 d^2 D}{\rho_t}+\frac{\beta \rho_t d}{20}+2L\beta d + \frac{\alpha}{20}, \frac{2L^2}{S_t}+2L\sqrt{\frac{2(\alpha+\beta \rho_t d)}{S_t}}\right\} \right.\nonumber\\
    &\qquad\qquad\left.+\, dD\rho_T + \alpha T + d\beta \sum_{t=1}^T\rho_t\right].
\end{align}
When $\alpha=\beta=0$, we denote the quantity above simply as $\opt^{0}_T(r)$. In this case, a single consistent choice of $\rho$ at all time steps suffices, i.e.,
\begin{align}
    \opt^0_{T}(r) = \inf_{\rho \ge r} \left[ \sum_{t=1}^{T}\min \left\{\frac{125L^2d^2D}{\rho}, \frac{2L^2}{S_t}\right\} + dD\rho\right].
\end{align}
This follows by noting that, for every feasible $r \le \rho_1 \le \dots \le \rho_T$, we can simply set $\rho_t=\rho_T$ for every $t$; this keeps the sequence feasible, and does not increase any individual term inside the summation.

Given this, a natural $\rho_t$ that the algorithm can choose, solely as a function of the information it has seen so far, is $\opt^0_{t-1}(r)$. This is the value of $\rho_t$ that would realize the best offline regret in hindsight, had the game been stopped before round $t$, and corresponds to an FTL move itself in the space of noise schedules. This choice of $\rho_t$ turns out to be non-decreasing; more interestingly, we can show that the regret of \Cref{alg:ftpl} instantiated with this choice of $\rho_t$ is constant-factor competitive with $\opt^{\alpha,\beta}_T(r)$.

\begin{restatable}[FTL for Noise Scale Schedule]{theorem}{metaFTL}
    \label{thm:meta-ftl}
    Fix $r >0$. Let $\rho_1=r$, and for $t \ge 2$,
    \begin{align}
        \label{eqn:rho-FTL}
        \rho_{t} = \min \arginf_{\rho \ge r} \left[ \sum_{i=1}^{t-1}\min \left\{\frac{125L^2d^2D}{\rho}, \frac{2L^2}{S_i}\right\} + dD\rho\right],
    \end{align}
    Then,
    \begin{enumerate}
        \item $\rho_1 \le \dots \le \rho_T$, and every $\rho_t$ depends only on $S_1,\dots,S_{t-1}$.
        \item For this noise scale schedule, assuming that the error parameter $\beta \le c/T$, \Cref{alg:ftpl} has the following regret guarantee:
        \begin{align*}
            \E_\sigma\left[\regret_T(f_{1:T};x_{1:T})\right]\le C_1 \cdot \opt^{0}_T(r) \le C_2 \cdot \opt^{\alpha, \beta}_T(r).
        \end{align*}
         Here, $C_1$ and $C_2$ are constants depending only on $d, D, L, r$, and $c$.
    \end{enumerate}
\end{restatable}
We note that $\rho_t$ as given by \eqref{eqn:rho-FTL} can be efficiently computed, provided that we are given the values $S_1,\dots,S_{t-1}$. To see this, for each $S_i$, consider $b_i = \frac{125L^2d^2D}{2L^2/S_i}$ (where $b_i=0$ if $S_i=0$), and sort the $b_i$ values, so that $b_{i_1} \le b_{i_2} \le \dots \le b_{i_{t-1}}$. Furthermore, set $b_{i_0}=0$ and $b_{i_t}=\infty$, and consider the intervals $[\tilde{b}_{i_0},b_{i_1}),[\tilde{b}_{i_1},b_{i_2}),\dots, [\tilde{b}_{i_{t-1}}, b_{i_t})$, where $\tilde{b}_{i_j}=\max(b_{i_j}, r)$. For any $\rho \in [\tilde{b}_{i_j}, b_{i_{j+1}})$, observe that the objective function in \eqref{eqn:rho-FTL} evaluates to
\begin{align*}
    \frac{125jL^2d^2D}{\rho} + dD\rho + \sum_{k < t:b_{i_k} > \rho} \frac{2L^2}{S_i}.
\end{align*}
Note that the last term is constant for every $\rho \in [\tilde{b}_{i_j}, b_{i_{j+1}})$. Thus, the ``unconstrained'' minimizer for this interval balances the first two terms, giving 
\begin{align*}
    \hat{\rho}_j = \sqrt{125jL^2d}.
\end{align*}
This is a valid value for $\hat{\rho}_j$, provided $\sqrt{125jL^2d} \in [\tilde{b}_{i_j}, b_{i_{j+1}})$. However, if $\sqrt{125jL^2d} \notin [\tilde{b}_{i_j}, b_{i_{j+1}})$, we simply set $\hat{\rho}_j$ to be either $\tilde{b}_{i_j}$ or $b_{i_{j+1}}$, based on whichever of the two evaluates to a smaller objective in \eqref{eqn:rho-FTL}. Finally, we set $\rho_t$ to be the $\hat{\rho}_j$ value which realizes the smallest objective among all $\hat{\rho}_j$'s.

We give the complete proof of \Cref{thm:meta-ftl} in \Cref{sec:proofs-adaptive-ftl}, which goes through a be-the-leader analysis again. Namely, it considers the regret of the algorithm, had it chosen $\rho_{t+1}$ instead of $\rho_t$ at time step $t$ (it cannot do so, since it does not know $f_t$). Thereafter, it bounds the cost of instead having played $\rho_t$, by using the specific form of the objective function.

\Cref{thm:meta-ftl} allows us to derive precise regret guarantees of \Cref{alg:ftpl}, with the choice of $\rho_t$ given in \eqref{eqn:rho-FTL}, for different curvature regimes.

\begin{restatable}[Curvature-adaptive Regret]{corollary}{regimeSpecificRates}
    \label{corollary:precise-regimes}
    Consider \Cref{alg:ftpl} instantiated with the noise schedule $\rho_t$ given in \eqref{eqn:rho-FTL}, and suppose that $\alpha,\beta = O(1/T)$. Then, the following guarantee always holds:
    \begin{align*}
        \E_{\sigma}[\regret_T(f_{1:T};x_{1:T})] \le O(\sqrt{T}).
    \end{align*}
    Additionally, if the cumulative curvature satisfies $S_t \ge t^{\gamma}$, we have that
    \begin{align*}
        \E_{\sigma}[\regret_T(f_{1:T};x_{1:T})] &=
        \begin{cases}
            O(\sqrt{T}), & \gamma \le \frac12, \\
            O(T^{1-\gamma}), & \frac12 < \gamma < 1, \\
            O(\log T), & \gamma = 1, \\
            O(1), & \gamma > 1.
        \end{cases}
    \end{align*}
\end{restatable}
The proof of \Cref{corollary:precise-regimes} is a calculation upper-bounding $\opt^{\alpha, \beta}_T(r)$  in each of the regimes, and is given in \Cref{sec:proofs-adaptive-ftl}.

It is instructive to compare the regime-wise guarantees above to the regime-wise guarantees given in Corollary 3.3 in \citet{bartlett2007adaptive} for online \textit{convex} optimization. Specifically, \citet{bartlett2007adaptive} consider the setting where every individual $f_t$ has curvature $t^{-\lambda}$ (which implies that $S_t \approx t^{1-\lambda}$), and do casework on $\lambda$. Mapping their cases to our cases above yields identical regret bounds as theirs. In this way, \Cref{corollary:precise-regimes} generalizes their guarantees by analyzing different regimes of \textit{cumulative} curvature, which allows non-convex functions too.

Lastly, we comment about the setting of the error parameters $\alpha, \beta$ in our results. \Cref{thm:meta-ftl} required no condition on $\alpha$, but required $\beta=O(1/T)$. On the other hand, \Cref{corollary:precise-regimes} required $\alpha=O(1/T)$ as well. Similar conditions on the error parameters were required by \cite{suggala2020online} as well. In particular, \cite{suggala2020online} required setting $\alpha=O(1/\sqrt{T})$ and $\beta=O(1/T)$ in order to guarantee $O(\sqrt{T})$ regret. This suffices for us as well for the purposes of getting a uniform $O(\sqrt{T})$ regret guarantee. However, since we are aiming for $o(\sqrt{T})$ regret in the cases where curvature is ample, we require $\alpha$ to be smaller. This is made clear by considering the expression \eqref{eqn:offline-opt-rho} for $\opt^{\alpha, \beta}_T(r)$, which explicitly contains an $\alpha T$ term.

\section{Lower Bound}
\label{sec:lower-bound}

The previous section showed that the regret of our algorithm, instantiated with the suggested noise scale schedule, is at most $O(\opt^0_T(r))$. In fact, it holds that $\opt^{0}_T(r) \le \opt^{0}_T(0) + dDr$ (see \Cref{claim:relating-opt-0-r-to-opt-0-0}), which means that, holding $r,D,L$ fixed, the regret of our algorithm grows with $T$ as $O(\opt^{0}_T(0))$. We will now show that a regret of $\Omega(\opt^{0}_T(0))$ is necessary in the worst case for \textit{any} algorithm. This shows that the $O(\opt^{0}_T(0))$ regret bound that we obtained above is not just a consequence of our analysis, but is in fact inherent. Moreover, it also shows that our algorithm achieves a minimax optimal regret.

Our lower bound holds for any given sequence of cumulative curvatures satisfying $0 = s_0 \le s_1 \le s_2 \le \dots \le s_T$. In particular, this means that an $\Omega(\opt^{0}_T(0))$ regret is necessary  already for a sequence of \textit{convex} functions.  Given such a curvature sequence, we derive a sequence of functions $f_1,\dots,f_T$ that are $L$-Lipschitz on $\mcX$, and have cumulative curvature $S_t=s_t$ for all $t$. %

More concretely, define the minimax regret for a cumulative curvature sequence $s_{1:T}=s_1,\dots,s_T$ to be
\begin{align}
    \mfR_T(s_{1:T}) := \inf_{\mcA} \sup_{\substack{f_1,\ldots,f_T\text{ } L\text{-Lipschitz}\text{ on }\mcX \\ S_t=s_t\text{ for all }t}}
  \E_{\mcA} \left[\regret_T(f_{1:T};x_{1:T})\right].
\end{align}

\begin{restatable}[Curvature-specific Lower Bound]{theorem}{lowerBound}
    \label{thm:lower-bound}
    Fix the dimension $d=1$. There is a universal constant $c > 0$, such that for every $T \ge 1$ and every non-decreasing sequence $0=s_0 \le s_1 \le \dots \le s_T$ satisfying $s_t-s_{t-1} \le \frac{L}{2D}$ for all $t \in [T]$,
    \begin{align}
        \mfR_T(s_{1:T}) &\ge c \cdot \inf_{\rho \ge 0} \left[\sum_{t=1}^{T}\min \left\{\frac{L^2D}{\rho}, \frac{L^2}{s_t}\right\} + D\rho\right] = \Omega\left( \opt^{0}_T(0) \right).
    \end{align}
\end{restatable}
As discussed above, since our algorithm always has a regret guarantee of $O(\opt^{0}_T(0))$, we thus have that the minimax regret satisfies $\mfR_T(s_{1:T})=\Theta(\opt^{0}_T(0))$.

The entire proof of \Cref{thm:lower-bound} comprises of several intricate steps, and we defer the details to \Cref{sec:proofs-lower-bound}. Here, we highlight some salient aspects of the construction. Previous work by \citet{abernethy2008optimal} shows a minimax lower bound of $\Omega(\sqrt{T})$ for generic online convex optimization, as well as a lower bound of $\Omega(\log T)$ for strongly convex online optimization. These constructions use (1) a sequence of random linear functions for the general convex case, and (2) a sequence of quadratics with drifting centers for the strongly convex case. 

Neither of these constructions, by themselves, directly extend to our case, where we wish to show a lower bound for an arbitrary sequence of non-negative and non-decreasing cumulative curvatures. Nevertheless, a natural strategy would be to carefully mix in both the above components. That is, we choose every $f_t$ to comprise of a random linear term, together with a quadratic function whose center $z_t$ follows a suitably chosen random process. The curvature constraints imply that the coefficient on the quadratic term ought to be $a_t=s_t-s_{t-1}$. The main technical innovation in the analysis is the construction of fictitious functions $G_t(x)$ that are regularized forms of the cumulative functions $F_t(x)$. These functions enable lower-bounding the regret of any algorithm by relating the best-in-hindsight loss on $F_T$ to the telescoping summation of $G_t(z_t)-G_{t-1}(z_{t-1})$.

Finally, in \Cref{sec:regime-specific-lower-bounds}, we complement the regime-specific regret upper bounds given in \Cref{corollary:precise-regimes} with matching lower bounds. This shows that the guarantees for our algorithm given in \Cref{corollary:precise-regimes} are optimal.
\section{Simulations}
\label{sec:simulations}

\begin{figure}[t]
	\begin{subfigure}{.245\textwidth}
	\centering
	\includegraphics[width=1\linewidth]{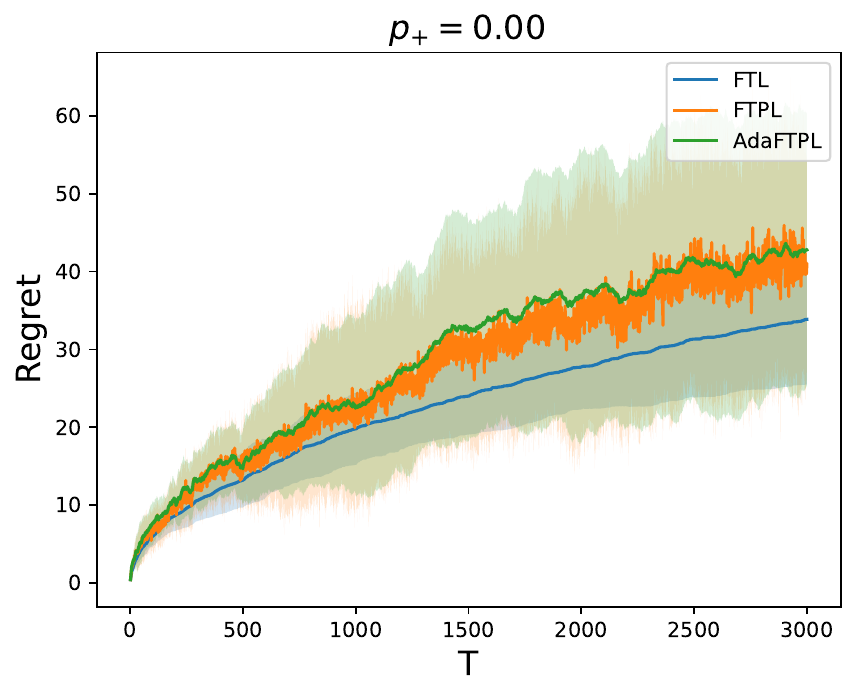}  
	\caption{}
	\label{fig:ftl_d=10_0.00}
\end{subfigure}
	\begin{subfigure}{.245\textwidth}
	\centering
	\includegraphics[width=1\linewidth]{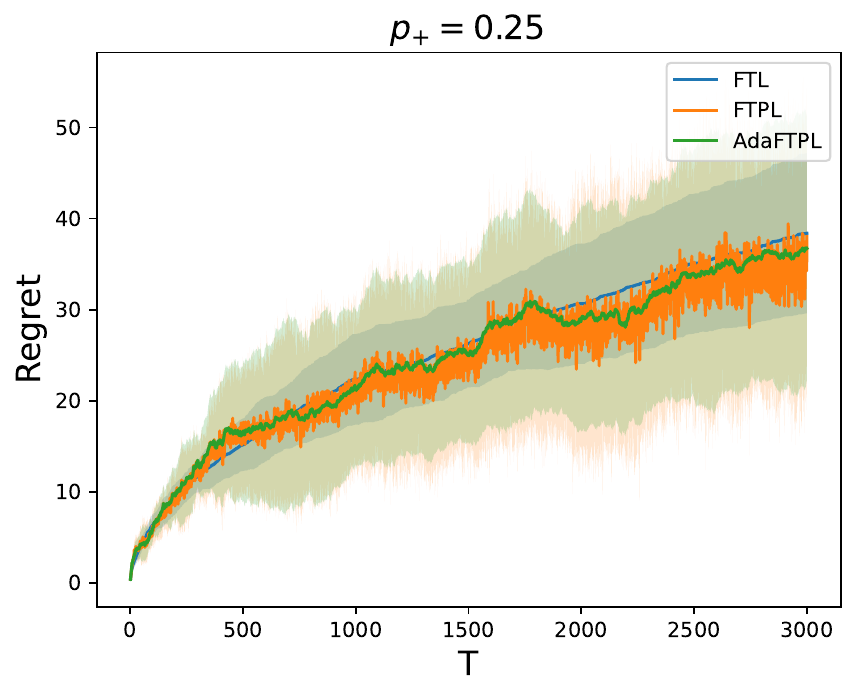}  
	\caption{}
	\label{fig:ftl_d=10_0.25}
\end{subfigure}
	\begin{subfigure}{.245\textwidth}
	\centering
	\includegraphics[width=1\linewidth]{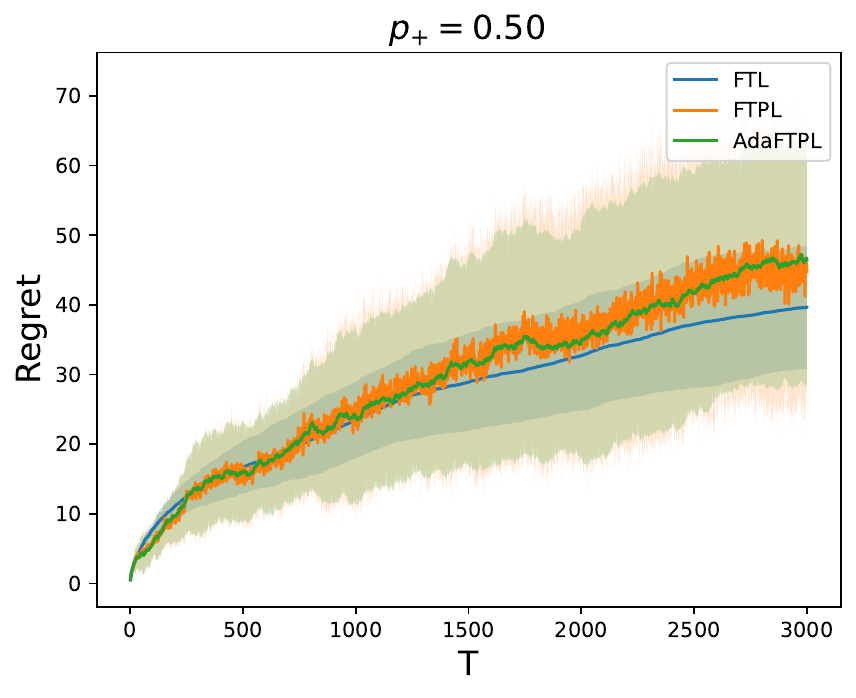}  
	\caption{}
	\label{fig:ftl_d=10_0.50}
\end{subfigure}
	\begin{subfigure}{.245\textwidth}
	\centering
	\includegraphics[width=1\linewidth]{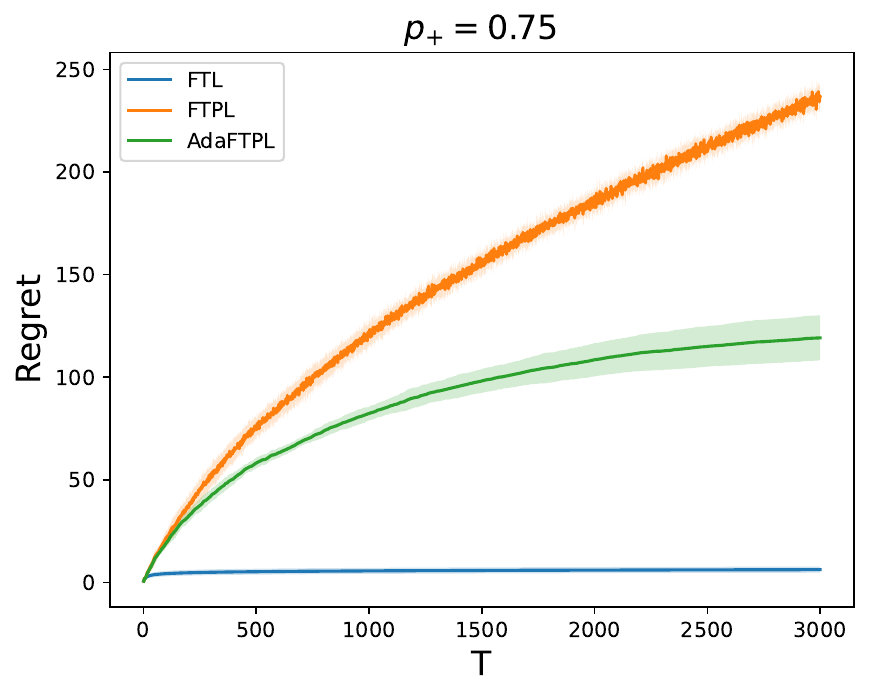}  
	\caption{}
	\label{fig:ftl_d=10_0.75}
\end{subfigure}
        \newline
    \begin{subfigure}{.245\textwidth}
	\centering
	\includegraphics[width=1\linewidth]{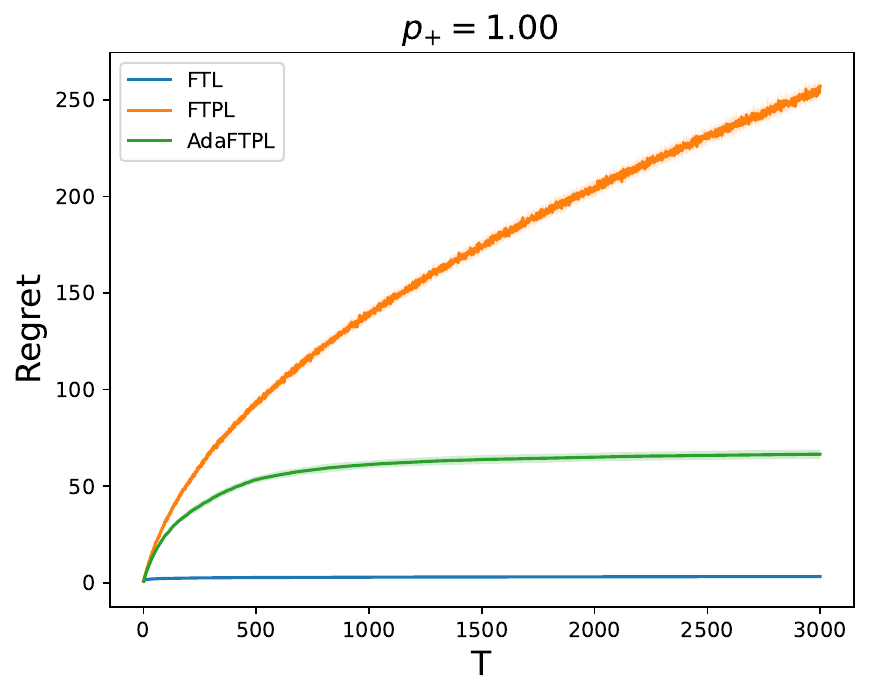}  
	\caption{}
	\label{fig:ftl_d=10_1.00}
\end{subfigure}
	\begin{subfigure}{.245\textwidth}
	\centering
	\includegraphics[width=1\linewidth]{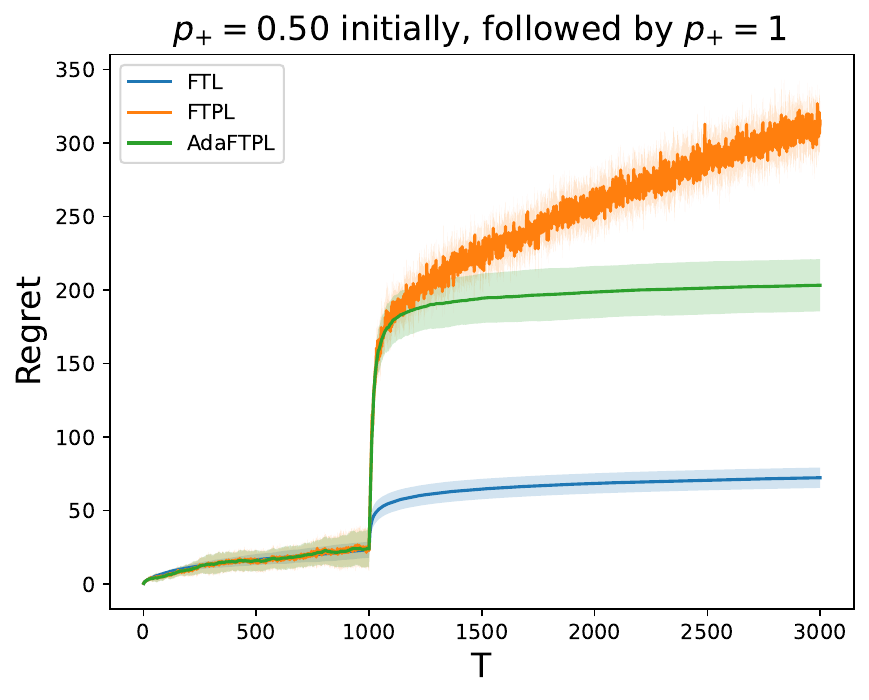}  
	\caption{}
	\label{fig:ftl_middle_d=10_0.50}
\end{subfigure}
	\begin{subfigure}{.245\textwidth}
	\centering
	\includegraphics[width=1\linewidth]{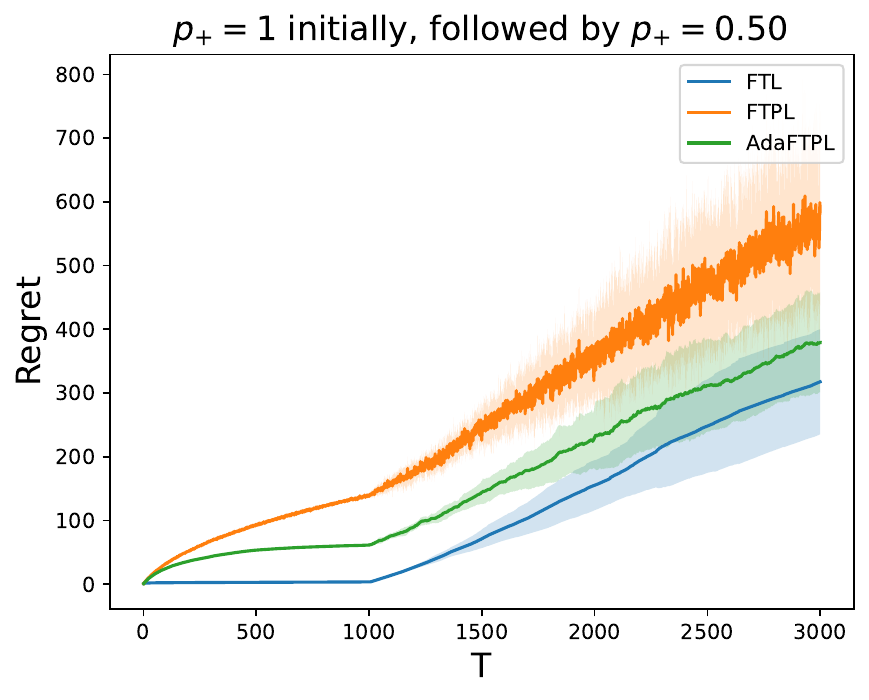}  
	\caption{}
	\label{fig:ftl_reverse_middle_d=10_0.50}
\end{subfigure}
	\begin{subfigure}{.245\textwidth}
	\centering
	\includegraphics[width=1\linewidth]{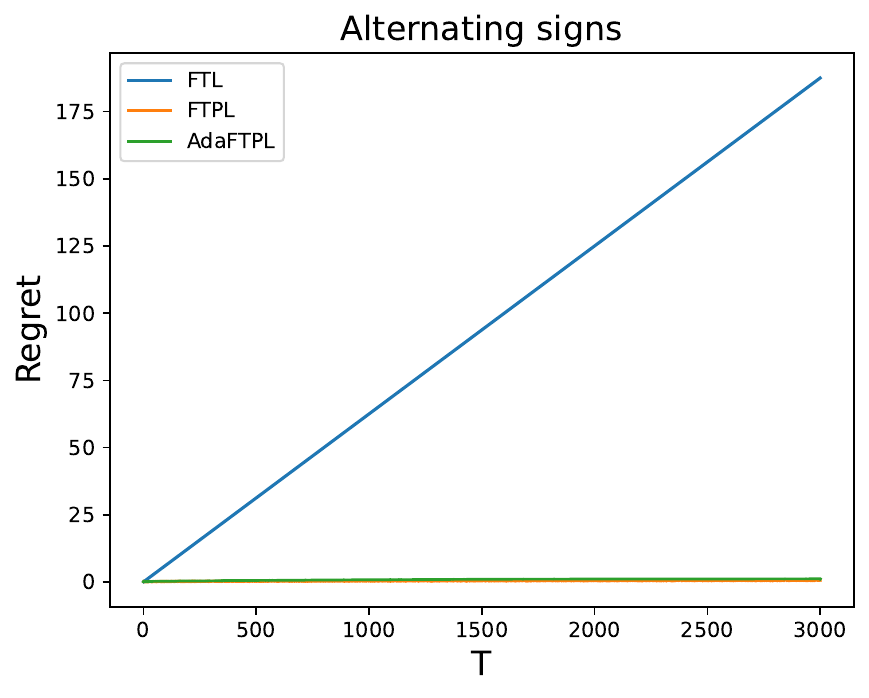}  
	\caption{}
	\label{fig:ftl_interleaving_d=1}
\end{subfigure}
	\caption{(a)-(e) We gradually increase curvature. (f),(g) Curvature is introduced/taken away in the middle. (h) Adversarial sequence that causes FTL to oscillate heavily and incur linear regret.}
	\label{fig:plots-quadratics}
\end{figure}

We perform synthetic experiments on a sequence of quadratic functions, and compare the regret of three algorithms: FTL, FTPL and AdaFTPL. We specifically choose quadratics to have precise control on the cumulative curvatures, and also because the optimization oracle for all the algorithms can be computed exactly in closed form. Concretely, we set $d=10$, and consider $\mcX=[-D/2,D/2]^d$. Each $f_t$ is of the form
\begin{align}
    \label{eqn:form-f-t}
    f_t(x) := \kappa_t \sum_{i=1}^d a_{t,i}(x_i - b_{t,i})^2,
\end{align}
where $\kappa_t \in \{+1,-1\}$ is set to be $+1$ with probability $p_+$, and $-1$ otherwise. Each $a_{t,i} \sim [0,1/2D]$, and $b_{t,i} \sim [-D/2,D/2]$ uniformly at random, which ensures that $f_t$ is 1-Lipschitz on $\mcX$. Note that if $p_+=0$, every $f_t$ is non-convex, and as $p_+$ increases, the cumulative curvature builds up, with all the functions being strongly convex at $p_+=1$. We consider horizons ranging from $T=1$ to $T=3000$. For each horizon $T$, we run each of the algorithms on the sequence $f_1,\dots,f_T$, and compute its final regret. Recall that FTPL as given by \citet{suggala2020online} is simply AdaFTPL with $\rho_t=\rho$ for all $t$. We set $\rho=\sqrt{T}$ as suggested by \citet{suggala2020online}, which guarantees an $O(\sqrt{T})$ regret bound. For AdaFTPL, we set $\rho_t$ according to the noise scale schedule in \Cref{thm:meta-ftl}. %
Figures \ref{fig:ftl_d=10_0.00}-\ref{fig:ftl_d=10_1.00} plot the regret of the algorithms (averaged over 10 random runs) against the time horizon $T$.

We can observe that when $p_+$ is small, all the three algorithms perform similarly. As $p_+$ increases, and curvature starts showing up in the function sequence, AdaFTPL adapts to this, and achieves much better regret than FTPL. In fact, in the cases where $p_+=0.75$ and $p_+=1$, the regret of AdaFTPL grows similarly to the regret of FTL ($\sim \log T$), whereas the regret of FTPL grows significantly worse ($\sim \sqrt{T}$).

To highlight the adaptive behavior of AdaFTPL, we also consider: (1) a sequence where $p_+=0.5$ for the first $1000$ functions, and $p_+=1$ for the rest of the functions, and (2) a sequence where $p_+=1$ for the first $1000$ functions, and $p_+=0.5$ for the rest of the functions. The coefficients $a_{t,i}$ for the first 1000 functions are scaled down; the idea is to either introduce/take away positive curvature in the middle of the sequence, and see if AdaFTPL adapts to it as expected. We can see in Figures \ref{fig:ftl_middle_d=10_0.50} and \ref{fig:ftl_reverse_middle_d=10_0.50} that this is indeed so: in the first case, beyond $T = 1000$, the regret of AdaFTPL grows akin to FTL, and much slower than FTPL. In the latter case, AdaFTPL exploits the available curvature in the first 1000 steps, and thereafter still maintains its conservative regret guarantee.

Finally, we also include an example (Figure \ref{fig:ftl_interleaving_d=1}) of an adversarial sequence of functions (with $d=1$) that has alternating $\kappa_t$ signs, a standard and well-known failure case for FTL. For this sequence, FTL plays a sequence of moves that oscillates a lot, causing linear regret, whereas both FTPL and AdaFTPL, by virtue of their perturbations, play stable moves that result in sublinear regret. This example highlights how FTL fails \textit{drastically} in the worst case, even if it may work well in other problem instances.

\section{Conclusion}
\label{sec:conclusion}

In this work, we derived an Adaptive FTPL algorithm that uses time-varying perturbations, and enjoys regret guarantees that interpolate between $O(\log T)$ and $O(\sqrt{T})$ based on the curvature present in the function sequence. We also obtained an intuitive and efficient method to set the noise scale of the perturbation at each round, as a function of the curvatures of the functions seen so far. We give a matching lower bound construction which shows that our regret guarantees are optimal. Our work thus extends adaptive regret rates in the theory of online optimization significantly beyond online convex optimization.

\subsection*{Acknowledgements}
This work is supported by Moses Charikar's and Gregory Valiant's Simons Investigator Awards, and a Google PhD Fellowship. ChatGPT was routinely used in the development and understanding of ideas used in the paper

\bibliographystyle{plainnat} 
\bibliography{references}

\appendix
\section{Proofs from \Cref{sec:adaptive-ftpl}}
\label{sec:proofs-adaptive-ftl}

We restate and prove \Cref{thm:main}.
\mainTheorem*
\begin{proof}
    By a standard argument for FTPL (e.g., see Lemma 12 in \citet{hutter2005adaptive}, Chapter 4 in \citet{Cesa-Bianchi_Lugosi_2006}), it suffices to consider the setting where the functions are specified upfront by an \textit{oblivious} adversary, and show the regret bound for every fixed $f_1,\dots,f_T$. With obliviously specified functions, it furthermore suffices in \Cref{alg:ftpl} to simply draw a random vector $\sigma=\{\sigma_j\}_{j=1}^d \stackrel{\mathrm{i.i.d.}}{\sim} \text{Exp}(1)$ once at the start and reuse it at every time step, instead of drawing a fresh $\sigma_t$ at each step. With these considerations, our goal is to upper bound the regret in the left-hand side of \eqref{eqn:main-regret-bound} for fixed $f_1,\dots,f_T$, where the expectation is only with respect to $\sigma$ drawn at the start of the game.

    We begin by defining a key quantity in the analysis: for any $t=1,\dots,T$, let $\bar{x}_{t+1}= \mcO_{\alpha, \beta}(F_{t}, \rho_t \sigma)$. Note that $\bar{x}_{t+1}$ is separate from $x_{t+1}=\mcO_{\alpha, \beta}(F_t,\, \rho_{t+1}\sigma)$ that the algorithm plays at time step $t+1$. We introduce $\bar{x}_{t+1}$, which uses the noise vector from the \textit{previous} time step, solely for the sake of analysis. This will help us (1) establish a standard decomposition of the regret into a \textit{stability} and \textit{perturbation} term, and (2) directly utilize the bound on the stability term from \citet{suggala2020online}.

    We now claim that\footnote{We note that setting $\rho_1=\dots=\rho_T=\frac{1}{\eta}$ recovers the bound in \citet{suggala2020online}.}:
    \begin{align}
        \label{eqn:stability-perturbation-decomposition}
        \E_\sigma\left[\sum_{t=1}^T f_t(x_t) - \inf_{x \in \mcX} \sum_{t=1}^T f_t(x)\right] \le \underbrace{L \sum_{t=1}^T \E_\sigma \left[ \left\| x_t - \bar{x}_{t+1}\right\|_1 \right]}_{\text{stability}} + \underbrace{d D\rho_T+ \alpha T +  d\beta \sum_{t=1}^T \rho_t.}_{\text{perturbation}}
    \end{align}
    To see this, fix any $x^\star \in \mcX$. We have that
    \begin{align}
        \sum_{t=1}^T f_t(x_t) - \sum_{t=1}^T f_t(x^\star) &= \sum_{t=1}^T(f_t(x_t)-f_t(\bar{x}_{t+1})) + \sum_{t=1}^T(f_t(\bar{x}_{t+1})-f_t(x^\star)) \nonumber\\
        &\le L \sum_{t=1}^T  \left\|x_t - \bar{x}_{t+1}\right\|_1 + \sum_{t=1}^T(f_t(\bar{x}_{t+1})-f_t(x^\star)),  \label{eqn:breakup}
    \end{align}
    where the inequality follows by the Lipschitzness assumption. We thus proceed to bound the second term in \eqref{eqn:breakup}. With a view to apply the ``be-the-leader'' inequality, define
    \begin{align*}
        g_t(x) := f_t(x) - \langle (\rho_{t}-\rho_{t-1})\sigma, x \rangle,
    \end{align*}
    (where $\rho_0=0$), and observe that
    \begin{align*}
        \sum_{i=1}^t g_i(x) = \sum_{i=1}^t f_i(x) - \langle \rho_t \sigma, x \rangle = F_t(x) - \langle \rho_t \sigma, x \rangle.
    \end{align*}
    So, by the definition of $\bar{x}_{t+1}$, we have that $\bar{x}_{t+1}$ is the (approximate) minimizer of $\sum_{i=1}^t g_i(x)$, where the summation crucially includes $i=t$. Namely, for any $t = 1,\dots,T$,
    \begin{align}
        \sum_{i=1}^t g_i(\bar{x}_{t+1}) &= F_t(\bar{x}_{t+1}) - \langle \rho_t \sigma, \bar{x}_{t+1} \rangle  \nonumber \\
        &\le \inf_{x \in \mcX} \left( F_t(x) - \langle \rho_t \sigma, x \rangle\right) + (\alpha + \beta\|\rho_t\sigma\|_1) = \inf_{x \in \mcX}\sum_{i=1}^t g_i(x) + (\alpha + \beta\|\rho_t\sigma\|_1). \label{eqn:be-the-leader-plugin}
    \end{align}
    We can thus instantiate the standard ``be-the-leader'' inequality to obtain that $\sum_{t=1}^{T} g_t(\bar{x}_{t+1}) \lesssim \inf_{x \in \mcX} \sum_{t=1}^T g_t(x)$. In more detail,
    \begingroup
    \allowdisplaybreaks
    \begin{align*}
        \sum_{t=1}^{T} g_t(\bar{x}_{t+1}) - \inf_{x \in \mcX} \sum_{t=1}^T g_t(x) &\le \sum_{t=1}^{T} g_t(\bar{x}_{t+1}) -\sum_{t=1}^T g_t(\bar{x}_{T+1}) + (\alpha + \beta\|\rho_T \sigma\|_1) \tag{plugging $t=T$ in \eqref{eqn:be-the-leader-plugin}} \\
        &= \sum_{t=1}^{T-1}g_t(\bar{x}_{t+1}) - \sum_{t=1}^{T-1} g_t(\bar{x}_{T+1}) + (\alpha + \beta\|\rho_T \sigma\|_1) \\
        &\le \sum_{t=1}^{T-1}g_t(\bar{x}_{t+1}) - \sum_{t=1}^{T-1} g_t(\bar{x}_{T}) + (\alpha + \beta\|\rho_{T-1} \sigma\|_1) + (\alpha + \beta\|\rho_T \sigma\|_1) \tag{plugging $t=T-1$ in \eqref{eqn:be-the-leader-plugin}} \\
        &= \sum_{t=1}^{T-2}g_t(\bar{x}_{t+1}) - \sum_{t=1}^{T-2} g_t(\bar{x}_{T}) + (\alpha + \beta\|\rho_{T-1} \sigma\|_1) + (\alpha + \beta\|\rho_T \sigma\|_1) \\
        &\;\;\vdots \tag{repeating all the way down}\\
        &\le g_1(\bar{x}_2) - g_1(\bar{x}_3) + \sum_{t=2}^{T}(\alpha + \beta\|\rho_t \sigma\|_1) \\
        &\le \sum_{t=1}^{T}(\alpha + \beta\|\rho_t \sigma\|_1)\tag{plugging $t=1$ in \eqref{eqn:be-the-leader-plugin}} \\
        &= \alpha T + \beta \|\sigma\|_1 \sum_{t=1}^T \rho_t. \tag{each $\rho_t \ge 0$}
    \end{align*}
    \endgroup
    Instantiating the above inequality with $x=x^\star$, we get
    \begin{align*}
        &\sum_{t=1}^{T} g_t(\bar{x}_{t+1}) \le \sum_{t=1}^{T} g_t(x^\star) + \alpha T + \beta \|\sigma\|_1 \sum_{t=1}^T \rho_t \\
        \implies \quad & \sum_{t=1}^T f_t(\bar{x}_{t+1}) - \sum_{t=1}^T\langle (\rho_{t}-\rho_{t-1})\sigma, \bar{x}_{t+1} \rangle \le \sum_{t=1}^T f_t(x^\star) - \sum_{t=1}^T\langle (\rho_{t}-\rho_{t-1})\sigma, x^\star \rangle + \alpha T + \beta \|\sigma\|_1 \sum_{t=1}^T \rho_t \\
        \implies \quad & \sum_{t=1}^T(f_t(\bar{x}_{t+1})-f_t(x^\star)) \le \sum_{t=1}^T\langle (\rho_t-\rho_{t-1})\sigma, \bar{x}_{t+1} - x^\star \rangle + \alpha T + \beta \|\sigma\|_1 \sum_{t=1}^T \rho_t \\
        \implies \quad & \sum_{t=1}^T(f_t(\bar{x}_{t+1})-f_t(x^\star)) \le \sum_{t=1}^T |\rho_t-\rho_{t-1}|\|\sigma\|_1 \|\bar{x}_{t+1} - x^\star\|_\infty + \alpha T + \beta \|\sigma\|_1 \sum_{t=1}^T \rho_t \tag{H\"{o}lder's inequality}\\
        \implies \quad & \sum_{t=1}^T(f_t(\bar{x}_{t+1})-f_t(x^\star)) \le D\|\sigma\|_1\rho_T + \alpha T + \beta \|\sigma\|_1 \sum_{t=1}^T\rho_t,
    \end{align*}
    where we used that $\mcX$ has $\ell_\infty$ diameter $D$, and that $\rho_0,\dots,\rho_T$ are non-negative and non-decreasing, with $\rho_0=0$. Substituting the bound above in \eqref{eqn:breakup}, and taking expectation with respect to $\sigma$, we get
    \begin{align*}
        \E_{\sigma} \left[ \sum_{t=1}^T f_t(x_t) - \sum_{t=1}^T f_t(x^\star)\right] &\le L \sum_{t=1}^T \E_\sigma \left[ \left\| x_t - \bar{x}_{t+1}\right\|_1 \right] + D\rho_T\E_\sigma\left[\|\sigma\|_1\right] + \alpha T +  \E_\sigma\left[\|\sigma\|_1\right] \cdot \beta \sum_{t=1}^T \rho_t \\
        &= L \sum_{t=1}^T \E_\sigma \left[ \left\| x_t - \bar{x}_{t+1}\right\|_1 \right] + d D\rho_T+ \alpha T +  d\beta \sum_{t=1}^T \rho_t,
    \end{align*}
    which was the claimed stability-perturbation decomposition in \eqref{eqn:stability-perturbation-decomposition}.

    Now, since $x_t$ and $\bar{x}_{t+1}$ both use the \textit{same} noise vector $\rho_t \sigma$ (whose entries are i.i.d. exponential random variables with parameter $1/\rho_t$), we can directly use the bound on the stability term established in \citet[Theorem 1]{suggala2020online}. Specifically, they show that:
    \begin{align}
        \E_\sigma \left[ \left\| x_t - \bar{x}_{t+1}\right\|_1 \right] &\le \frac{125Ld^2D}{\rho_t} + \frac{\beta \rho_t d}{20 L} + 2\beta d + \frac{\alpha}{20L}. \tag{NC}
    \end{align}

    On the other hand, we can also bound this term differently by considering the curvature of $F_t=\sum_{i=1}^t f_i$. Namely, let $S_t$ be the strong convexity constant of $F_t$ with respect to $\ell_1$ (where $S_t=0$ if $F_t$ is not strongly convex). Then note that the function $\Phi(x)=F_t(x)-\langle \rho_t\sigma, x\rangle$ also has strong convexity constant $S_t$. 
    
    Let $z_t \in \argmin_{x \in \mcX} \Phi(x)$. Since $\bar{x}_{t+1}$ also approximately minimizes $\Phi(x)$, we have that $\Phi(\bar{x}_{t+1})-\Phi(z_t) \le \alpha + \beta\|\rho_t \sigma\|_1$. Invoking the strong convexity inequality \eqref{eqn:strong-convexity} with $y=\bar{x}_{t+1}$ and $x=z_t$, and using that $\Phi(x)$ is minimized at $z_t$, we get
    \begin{align}
        \Phi(\bar{x}_{t+1}) \ge \Phi(z_t) + \frac{S_t}{2}\|\bar{x}_{t+1}-z_t\|_1^2 &\implies \, \frac{S_t}{2}\|\bar{x}_{t+1}-z_t\|_1^2 \le \Phi(\bar{x}_{t+1}) - \Phi(z_t) \le \alpha + \beta\|\rho_t \sigma\|_1 \nonumber \\
        &\implies \, \|\bar{x}_{t+1}-z_t\|_1 \le \sqrt{\frac{2(\alpha + \beta\|\rho_t \sigma\|_1)}{S_t}}. \label{eqn:triangle-ineq-1}
    \end{align}
    But note also that $x_t$ approximately minimizes $F_{t-1}(x)-\langle \rho_t\sigma, x\rangle=\Phi(x)-f_t(x)$. That is,
    \begin{align*}
        \Phi(x_t)-f_t(x_t) \le \inf_{x \in \mcX} (\Phi(x)-f_t(x)) + \alpha + \beta\|\rho_t \sigma\|_1 \le \Phi(z_t)-f_t(z_t) + \alpha + \beta\|\rho_t \sigma\|_1.
    \end{align*}
    This means that
    \begin{align}
        \label{eqn:lipschitzness-sandwich}
        \Phi(x_t)-\Phi(z_t) \le \alpha + \beta\|\rho_t \sigma\|_1 + f_t(x_t) - f_t(z_t) \le \alpha + \beta\|\rho_t \sigma\|_1 + L\|x_t-z_t\|_1,
    \end{align}
    where we used Lipschitzness in the last inequality. Invoking the strong convexity inequality \eqref{eqn:strong-convexity} again with $y=x_t$ and $x=z_t$, and using that $\Phi(x)$ is minimized at $z_t$, we get
    \begin{align}
        \Phi(x_t) \ge \Phi(z_t) + \frac{S_t}{2}\|x_t-z_t\|_1^2 &\implies \, \frac{S_t}{2}\|x_{t}-z_t\|_1^2 \le \Phi(x_t) - \Phi(z_t)  \nonumber \\
        &\implies \, \frac{S_t}{2}\|x_{t}-z_t\|_1^2 \le \alpha + \beta\|\rho_t \sigma\|_1 + L\|x_t-z_t\|_1 \tag{using \eqref{eqn:lipschitzness-sandwich}} \\
        &\implies \,\|x_t-z_t\|_1 \le \frac{L+\sqrt{L^2+2S_t(\alpha+\beta \|\rho_t \sigma\|_1)}}{S_t} \nonumber \\
        &\implies \,\|x_t-z_t\|_1 \le \frac{2L}{S_t} + \sqrt{\frac{2(\alpha + \beta\|\rho_t \sigma\|_1)}{S_t}}, \label{eqn:triangle-ineq-2}
    \end{align}
    where in the second-to-last step, we solved a quadratic inequality in the variable $\|x_t-z_t\|_1$, and in the last step, we used that $\sqrt{a+b} \le \sqrt{a} + \sqrt{b}$.

    Combining \eqref{eqn:triangle-ineq-1} and \eqref{eqn:triangle-ineq-2} with the triangle inequality, we get that
    \begin{align*}
        \|x_t - \bar{x}_{t+1}\|_1 \le \frac{2L}{S_t} + 2\sqrt{\frac{2(\alpha + \beta\|\rho_t \sigma\|_1)}{S_t}}.
    \end{align*}
    Taking expectation with respect to $\sigma$, we get
    \begin{align*}
        \E_\sigma \left[\left\|x_t - \bar{x}_{t+1}\right\|_1\right] &\le \frac{2L}{S_t} + 2\sqrt{\frac{2}{S_t}}\cdot \E_\sigma\left[ \sqrt{\alpha + \beta\|\rho_t \sigma\|_1}\right] \\
        &\le \frac{2L}{S_t} + 2\sqrt{\frac{2}{S_t}}\cdot \sqrt{\E_\sigma\left[ \alpha + \beta\|\rho_t \sigma\|_1\right]} \tag{Jensen's inequality} \\
        &= \frac{2L}{S_t} + 2\sqrt{\frac{2(\alpha + \beta\rho_t d)}{S_t}}. \tag{SC}
    \end{align*}
    Combining this strongly-convex stability bound (SC) with the fallback non-convex stability bound (NC) from above, and substituting into the stability-perturbation decomposition \eqref{eqn:stability-perturbation-decomposition} of the regret, we get that
    \begin{align*}
         \E_\sigma\left[\sum_{t=1}^T f_t(x_t) - \inf_{x \in \mcX} \sum_{t=1}^T f_t(x)\right] \le &\sum_{t=1}^T \min \left\{\frac{125L^2d^2D}{\rho_t} + \frac{\alpha+\beta \rho_t d}{20} + 2\beta d L, \frac{2L^2}{S_t}+2L\sqrt{\frac{2(\alpha + \beta\rho_t d)}{S_t}}\right\} \nonumber \\
         &\qquad + d D\rho_T+ \alpha T +  d\beta \sum_{t=1}^T \rho_t,
    \end{align*}
    which is the desired bound.
\end{proof}

We next restate and prove \Cref{thm:meta-ftl}.

\metaFTL*
\begin{proof}%
    Recall that we set $\rho_1=r$, and for $t \ge 2$,
    \begin{align*}
        \rho_{t} = \min \arginf_{\rho \ge r} \underbrace{\left[ \sum_{i=1}^{t-1}\min \left\{\frac{125L^2d^2D}{\rho}, \frac{2L^2}{S_i}\right\} + dD\rho\right]}_{:= \Psi_{t-1}(\rho)},
    \end{align*}
    where we defined the objective function inside the parentheses to be $\Psi_{t-1}(\rho)$. %
    It is clear that $\rho_t$ depends only on $S_1,\dots,S_{t-1}$. We will now argue that $\rho_1 \le \dots \le \rho_T$. Note that $r=\rho_1\le \rho_2$ holds immediately, since the minimization is over $\rho \ge r$.
    
    Now, fix $t \ge 3$, and consider any $y \in [r, \rho_{t-1})$. Since $\rho_{t-1}$ is the \textit{smallest} infimizer of $\Psi_{t-2}(\rho)$, it must be the case that
    \begin{align*}
        \Psi_{t-2}(y) > \Psi_{t-2}(\rho_{t-1}).
    \end{align*}
    Therefore,
    \begin{align*}
        \Psi_{t-1}(y) &= \Psi_{t-2}(y) + \min \left\{\frac{125L^2d^2D}{y},\frac{2L^2}{S_{t-1}}\right\} \\
        &> \Psi_{t-2}(\rho_{t-1}) + \min \left\{\frac{125L^2d^2D}{\rho_{t-1}},\frac{2L^2}{S_{t-1}}\right\} = \Psi_{t-1}(\rho_{t-1}),
    \end{align*}
    where we used that the function $\min \left\{\frac{125L^2d^2D}{\rho},\frac{2L^2}{S_{t-1}}\right\}$ is non-increasing in $\rho$. Thus, any feasible minimizer of $\Psi_{t-1}$ must at least be $\rho_{t-1}$, implying that $\rho_{t} \ge \rho_{t-1}$.

    Now, define the \textit{clean} offline optimum, namely the best regret upper bound in hindsight without any error terms, to be
    \begin{align}
        \label{eqn:clean-offline-opt}
        \opt^0_{T}(r) := \inf_{r \le \rho_1 \le \dots \le \rho_T} \left[ \sum_{t=1}^{T}\min \left\{\frac{125L^2d^2D}{\rho_t}, \frac{2L^2}{S_t}\right\} + dD\rho_T\right].
    \end{align}
    We claim that
    \begin{align}
        \label{eqn:clean-offline-opt-collapses}
        \opt^0_{T}(r) = \inf_{\rho \ge r} \left[ \sum_{t=1}^{T}\min \left\{\frac{125L^2d^2D}{\rho}, \frac{2L^2}{S_t}\right\} + dD\rho\right] = \inf_{\rho \ge r} \Psi_T(\rho).
    \end{align}
    This follows simply by noting that, for every feasible $r \le \rho_1 \le \dots \le \rho_T$, we can simply set $\rho_t=\rho_T$ for every $t$; this keeps the sequence feasible, and does not increase any individual term inside the summation.

    \paragraph{Be-the-leader on clean sequence.} We will now argue that the sequence of $\rho_t$'s above has small regret on the sequence of clean objective functions, compared to $\opt^0_{T}(r)$. This is the standard be-the-leader analysis, and proceeds as follows: let $g_0(\rho) := dD\rho$, and $g_t(\rho) = \min \left\{\frac{125L^2d^2D}{\rho}, \frac{2L^2}{S_t}\right\}$ for $t \ge 1$. By definition, for any $i \in \{1,\dots,T\}$:
    \begin{align}
        \label{eqn:ftl-substitution-condition}
       \sum_{t=0}^{i}g_t(\rho_{i+1}) = \Psi_{i}(\rho_{i+1}) = \inf_{\rho \ge r} \Psi_i(\rho),
    \end{align}
     and therefore,
     \begingroup
     \allowdisplaybreaks
     \begin{align*}
        \sum_{t=0}^T g_t(\rho_{t+1}) - \inf_{\rho \ge r} \Psi_T(\rho) &= \sum_{t=0}^T g_t(\rho_{t+1}) - \sum_{t=0}^{T}g_t(\rho_{T+1}) \\
        &= \sum_{t=0}^{T-1} g_t(\rho_{t+1}) - \sum_{t=0}^{T-1}g_t(\rho_{T+1}) \\
        &\le \sum_{t=0}^{T-1} g_t(\rho_{t+1}) - \sum_{t=0}^{T-1}g_t(\rho_{T}) \tag{from \eqref{eqn:ftl-substitution-condition}}\\
        &= \sum_{t=0}^{T-2} g_t(\rho_{t+1}) - \sum_{t=0}^{T-2}g_t(\rho_{T}) \\
        &\;\;\vdots \tag{repeating all the way down}\\
        &\le g_0(\rho_1) - g_0(\rho_2) \\
        &= dD(\rho_1-\rho_2) \\
        &= dD(r-\rho_2) \le 0. \tag{since $\rho_1=r$}
     \end{align*}
     \endgroup
     Thus, we get that
     \begin{align}
        g_0(\rho_1) + \sum_{t=1}^{T}g_t(\rho_{t+1}) - \opt^0_{T}(r) \le 0 \quad \implies\quad \sum_{t=1}^{T}g_t(\rho_{t+1}) \le \opt^0_{T}(r) - dDr. \label{eqn:be-the-leader-on-clean}
     \end{align}

     \paragraph{Actual regret on clean sequence.} Note that \eqref{eqn:be-the-leader-on-clean} benchmarks the performance of playing $\rho_{t+1}$ at time step $t$, but actually, we want to benchmark the performance of playing $\rho_t$. For this, observe that for any $t \ge 1$,
     \begin{align*}
        g_t(\rho_t) - g_t(\rho_{t+1}) &= \min \left\{\frac{125L^2d^2D}{\rho_t}, \frac{2L^2}{S_t}\right\} - \min \left\{\frac{125L^2d^2D}{\rho_{t+1}}, \frac{2L^2}{S_t}\right\} \\
        &\le \frac{125L^2d^2D}{\rho_t} - \frac{125L^2d^2D}{\rho_{t+1}},
     \end{align*}
     where we used that $\rho_t \le \rho_{t+1}$. Summing over $t$, the sum telescopes, and we get
     \begin{align*}
        \sum_{t=1}^T\left(g_t(\rho_t) - g_t(\rho_{t+1})\right) \le \frac{125L^2d^2D}{\rho_1}-\frac{125L^2d^2D}{\rho_{T+1}} \le \frac{125L^2d^2D}{r}.
     \end{align*}
     Substituting this in \eqref{eqn:be-the-leader-on-clean}, we get
     \begin{align*}
        \sum_{t=1}^T g_t(\rho_t) &\le \opt^0_{T}(r) - dDr + \frac{125L^2d^2D}{r}.
     \end{align*}
     Folding in the terminal term $dD\rho_T$ on both sides gives
     \begin{align*}
        dD\rho_T + \sum_{t=1}^T g_t(\rho_t) &\le \opt^0_{T}(r) - dDr + \frac{125L^2d^2D}{r} + dD\rho_T \\
        &\le \opt^0_{T}(r) - dDr + \frac{125L^2d^2D}{r} + dD\rho_{T+1} \\
        &\le \opt^0_{T}(r) - dDr + \frac{125L^2d^2D}{r} + dD\rho_{T+1} + \sum_{t=1}^Tg_t(\rho_{T+1}) \\
        &= 2\opt^0_{T}(r) - dDr + \frac{125L^2d^2D}{r}. \tag{by definition of $\rho_{T+1}$ and $\opt^0_{T}(r)$}
     \end{align*}
     Using $dDr \le \opt^0_{T}(r)$ further gives
     \begin{align}
        \label{eqn:ftl-clean-sequence-regret-ub}
        dD\rho_T + \sum_{t=1}^T g_t(\rho_t) \le C \cdot \opt^0_{T}(r),
     \end{align}
     where $C=\max\left(2, 1+\frac{125L^2d}{r^2}\right)$.

     \paragraph{Relating $\opt^0_{T}(r)$ to $\opt^{\alpha, \beta}_{T}(r)$.} 
     Recall the definition of $\opt^{\alpha, \beta}_{T}(r)$:
     \begin{align*}
        \opt^{\alpha,\beta}_T(r) := \inf_{r \le \rho_1 \le \dots \le \rho_T} &\left[\sum_{t=1}^T \min\left\{\frac{125 L^2 d^2 D}{\rho_t}+\frac{\beta \rho_t d}{20}+2L\beta d + \frac{\alpha}{20}, \frac{2L^2}{S_t}+2L\sqrt{\frac{2(\alpha+\beta \rho_t d)}{S_t}}\right\} \right.\nonumber\\
        &\qquad\qquad\left.+\, dD\rho_T + \alpha T + d\beta \sum_{t=1}^T\rho_t\right],
    \end{align*}
    and note that for every feasible $\rho_1 \le \dots \le \rho_T$, 
    \begin{align*}
        &\sum_{t=1}^T \min\left\{\frac{125 L^2 d^2 D}{\rho_t}+\frac{\beta \rho_t d}{20}+2L\beta d + \frac{\alpha}{20}, \frac{2L^2}{S_t}+2L\sqrt{\frac{2(\alpha+\beta \rho_t d)}{S_t}}\right\} + dD\rho_T + \alpha T + d\beta \sum_{t=1}^T\rho_t \\
        &\ge \sum_{t=1}^T g_t(\rho_t) + dD \rho_T \ge \sum_{t=1}^T g_t(\rho_T) + dD \rho_T \ge \opt^0_{T}(r),
    \end{align*}
    where we used that $g_t$ is non-increasing. Infimizing over all feasible $\rho_1 \le \dots \le \rho_T$, we get that
    \begin{align}
        \label{eqn:opt-alpha-beta-opt-clean-relation}
        \opt^{\alpha,\beta}_T(r) \ge \opt^0_{T}(r).
    \end{align}

    \paragraph{Upper-bounding the regret with respect to $\opt^{\alpha, \beta}_{T}(r)$.} Recall that we upper-bounded the ``clean terms'' of the regret of our $\rho_t$ schedule in terms of the clean offline optimum in \eqref{eqn:ftl-clean-sequence-regret-ub}. Furthermore, we upper-bounded the clean offline optimum with the true offline optimum in \eqref{eqn:opt-alpha-beta-opt-clean-relation}. It remains to fold in the ``non-clean'' terms that appear in the regret of our $\rho_t$ schedule, and upper-bound these in terms of $\opt^{\alpha, \beta}_{T}(r)$.

    Towards this, for $\rho \ge r$, consider the function
    \begin{align*}
        M_t(\rho) := \min\left\{\frac{125 L^2 d^2 D}{\rho}+\frac{\beta \rho d}{20}+2L\beta d + \frac{\alpha}{20}, \frac{2L^2}{S_t}+2L\sqrt{\frac{2(\alpha+\beta \rho d)}{S_t}}\right\}.
    \end{align*}
    Recall that $g_t(\rho) = \min \left\{\frac{125L^2d^2D}{\rho}, \frac{2L^2}{S_t}\right\}$. Suppose it is the case that $\frac{125L^2d^2D}{\rho}\le \frac{2L^2}{S_t}$. Then,
    \begin{align*}
        M_t(\rho) \le \frac{125 L^2 d^2 D}{\rho}+\frac{\beta \rho d}{20}+2L\beta d + \frac{\alpha}{20} = g_t(\rho)+\frac{\beta \rho d}{20}+2L\beta d + \frac{\alpha}{20}.
    \end{align*}
    On the other hand, suppose that $\frac{125L^2d^2D}{\rho} >  \frac{2L^2}{S_t}$. Then,
    \begin{align*}
        M_t(\rho) \le \frac{2L^2}{S_t}+2L\sqrt{\frac{2(\alpha+\beta \rho d)}{S_t}} = g_t(\rho)+2\sqrt{g_t(\rho)(\alpha+\beta \rho d)}.
    \end{align*}
    Therefore, the following bound always holds:
    \begin{align*}
        M_t(\rho) \le g_t(\rho)+\frac{\beta \rho d}{20}+2L\beta d + \frac{\alpha}{20} + 2\sqrt{g_t(\rho)(\alpha+\beta \rho d)}.
    \end{align*}
    Instantiating this inequality for our noise-scale schedule $\rho_1,\dots,\rho_T$, and summing, we get
    \begin{align*}
        \sum_{t=1}^T M_t(\rho_t) &\le \sum_{t=1}^T g_t(\rho_t) + \frac{\beta d}{20}\sum_{t=1}^T \rho_t + 2L\beta d T + \frac{\alpha T}{20} + 2\sum_{t=1}^T \sqrt{g_t(\rho_t)(\alpha+\beta \rho_t d)} \\
        &\le \sum_{t=1}^T g_t(\rho_t) + \frac{\beta d}{20}\sum_{t=1}^T \rho_t + 2L\beta d T + \frac{\alpha T}{20} + 2\sqrt{\left(\sum_{t=1}^Tg_t(\rho_t) \right)\left(\sum_{t=1}^T(\alpha + \beta \rho_t d)\right)}. \tag{Cauchy-Schwarz inequality}
    \end{align*}
    Adding in the perturbation terms from \eqref{eqn:stability-perturbation-decomposition}, we get that
    \begin{align}
        &\sum_{t=1}^T M_t(\rho_t) + dD\rho_T + \alpha T + d\beta \sum_{t=1}^T \rho_t \nonumber \\
        &\le \sum_{t=1}^T g_t(\rho_t) + dD\rho_T + \frac{21\beta d}{20}\sum_{t=1}^T \rho_t + 2L\beta d T + \frac{21\alpha T}{20} + 2\sqrt{\left(\sum_{t=1}^Tg_t(\rho_t) \right)\left(\sum_{t=1}^T(\alpha + \beta \rho_t d)\right)} \nonumber \\
        &\le \underbrace{\sum_{t=1}^T g_t(\rho_t) + dD\rho_T}_{\text{clean regret from \eqref{eqn:ftl-clean-sequence-regret-ub}}} + \frac{21\beta d}{20}\sum_{t=1}^T \rho_t + 2L\beta d T + \frac{21\alpha T}{20} + 2\sqrt{\underbrace{\left(\sum_{t=1}^Tg_t(\rho_t) +dD\rho_T\right)}_{\text{clean regret from \eqref{eqn:ftl-clean-sequence-regret-ub}}}\left(\sum_{t=1}^T(\alpha + \beta \rho_t d)\right)} \nonumber \\
        &\le C \cdot \opt^{\alpha,\beta}_T(r) + \underbrace{\frac{21\beta d}{20}\sum_{t=1}^T \rho_t + 2L\beta d T + \frac{21\alpha T}{20}}_{:=\, E_T} + 2\sqrt{C \cdot \opt^{\alpha,\beta}_T(r)\left(\sum_{t=1}^T(\alpha + \beta \rho_t d)\right)} \nonumber \\
        &\le C \cdot \opt^{\alpha,\beta}_T(r) + E_T + 2\sqrt{C \cdot \opt^{\alpha,\beta}_T(r)\cdot E_T}. \label{eqn:regret-upper-bound-with-error-terms}
    \end{align}
    \paragraph{Bounding $E_T$.} Finally, we aim to upper-bound $E_T$ in terms of $\opt^{\alpha,\beta}_T(r)$. We have
    \begin{align*}
        E_T &= \frac{21\beta d}{20}\sum_{t=1}^T \rho_t + 2L\beta d T + \frac{21\alpha T}{20} \\
        &\le \frac{21\beta d T \rho_T}{20} + 2L\beta d T + \frac{21\alpha T}{20} \tag{since $\rho_t$ is non-decreasing}.
    \end{align*}
    Now, observe that $\alpha T \le \opt^{\alpha,\beta}_T(r)$. Furthermore, from \eqref{eqn:ftl-clean-sequence-regret-ub} and \eqref{eqn:opt-alpha-beta-opt-clean-relation},
    \begin{align*}
        dD\rho_T \le dD\rho_T + \sum_{t=1}^T g_t(\rho_t) \le C \cdot \opt^{0}_T(r) \le C \cdot \opt^{\alpha,\beta}_T(r).
    \end{align*}
    Finally, since it also holds that $dDr \le \opt^{\alpha,\beta}_T(r)$, we have that $L\beta d T \le \frac{L\beta T}{Dr}\cdot \opt^{\alpha,\beta}_T(r)$. Combining all this, we get
    \begin{align*}
        E_T &\le \left(\frac{21C\beta T}{20D}  + \frac{2L\beta T}{Dr} + \frac{21}{20}\right)\cdot \opt^{\alpha,\beta}_T(r).
    \end{align*}

    \paragraph{Plugging in $\beta=O(1/T)$.}
    Note that $\beta \le c/T$ suffices to get $E_T \le C' \cdot \opt^{\alpha,\beta}_T(r)$, where $C'$ is a constant depending on $L, D, d, r$ and $c$. Finally, plugging this back in \eqref{eqn:regret-upper-bound-with-error-terms}, we get
    \begin{align*}
        \sum_{t=1}^T M_t(\rho_t) + dD\rho_T + \alpha T + d\beta \sum_{t=1}^T \rho_t \le C'' \cdot \opt^{\alpha,\beta}_T(r),
    \end{align*}
    where $C''$ is a constant depending on $L, D, d, r$ and $c$. Recalling that the quantity on the left-hand side above is precisely the upper-bound on regret shown in \Cref{thm:main} completes the proof.
\end{proof}

\regimeSpecificRates*
\begin{proof}%
    In view of \Cref{thm:meta-ftl}, it suffices to upper-bound $\opt^{\alpha, \beta}_T(r)$, where
     \begin{align*}
        \opt^{\alpha,\beta}_T(r) := \inf_{r \le \rho_1 \le \dots \le \rho_T} &\left[\sum_{t=1}^T \min\left\{\frac{125 L^2 d^2 D}{\rho_t}+\frac{\beta \rho_t d}{20}+2L\beta d + \frac{\alpha}{20}, \frac{2L^2}{S_t}+2L\sqrt{\frac{2(\alpha+\beta \rho_t d)}{S_t}}\right\} \right.\nonumber\\
        &\qquad\qquad\left.+\, dD\rho_T + \alpha T + d\beta \sum_{t=1}^T\rho_t\right].
    \end{align*}
    Choose $\rho_t=\rho$ for all $t$. By taking the non-convex branch in the min, we have that
    \begin{align*}
        \opt^{\alpha,\beta}_T(r) &\le \sum_{t=1}^T\left(\frac{125 L^2 d^2 D}{\rho}+\frac{\beta \rho d}{20}+2L\beta d + \frac{\alpha}{20}\right) + dD\rho + \alpha T + d\beta T \rho \\
        &=  \frac{125 L^2 d^2 DT}{\rho}+\frac{21\beta \rho dT}{20}+2L\beta dT + dD\rho + \frac{21\alpha T}{20}  \\
        &= O\left(\frac{T}{\rho}+\rho\right) \tag{substituting $\alpha, \beta = O(1/T)$}.
    \end{align*}
    Setting $r=1$ and $\rho = \sqrt{T}$
    gives $\opt^{\alpha,\beta}_T(r) \le O(\sqrt{T})$.

    Now, suppose that $S_t \ge t^\gamma$. When $\gamma \le 1/2$, we already have the $O(\sqrt{T})$ regret bound from above.
    
    \paragraph{Regime $\frac12 < \gamma < 1$.}
    Choose $r=1$ and $\rho_t=\rho=1$ for all $t$. By taking the strongly-convex branch in the min, we have that
    \begin{align*}
        \opt^{\alpha,\beta}_T(r) &\le \sum_{t=1}^T\left(\frac{2L^2}{S_t}+2L\sqrt{\frac{2(\alpha+\beta \rho d)}{S_t}}\right) + dD\rho + \alpha T + d\beta T \rho \\
        &\le O\left(\sum_{t=1}^T t^{-\gamma}+ T^{-1/2}\sum_{t=1}^Tt^{-\gamma/2}\right) \\
        &\le O \left(T^{1-\gamma} + T^{1/2-\gamma/2}\right) = O(T^{1-\gamma}).
    \end{align*}

    \paragraph{Regime $\gamma = 1$.}
    Choose $r=1$ and $\rho_t=\rho=1$ for all $t$. By taking the strongly-convex branch in the min, we have that
    \begin{align*}
        \opt^{\alpha,\beta}_T(r) &\le \sum_{t=1}^T\left(\frac{2L^2}{S_t}+2L\sqrt{\frac{2(\alpha+\beta \rho d)}{S_t}}\right) + dD\rho + \alpha T + d\beta T \rho \\
        &= O\left(\sum_{t=1}^T\frac{1}{t}+T^{-1/2}\sum_{t=1}^Tt^{-1/2}\right) = O(\log T).
    \end{align*}

    \paragraph{Regime $\gamma > 1$.}
    Choose $r=1$ and $\rho_t=\rho=1$ for all $t$. By taking the strongly-convex branch in the min, we have that
        \begin{align*}
            \opt^{\alpha,\beta}_T(r) &\le \sum_{t=1}^T\left(\frac{2L^2}{S_t}+2L\sqrt{\frac{2(\alpha+\beta \rho d)}{S_t}}\right) + dD\rho + \alpha T + d\beta T \rho \\
            &= O\left(\sum_{t=1}^Tt^{-\gamma} + T^{-1/2}\sum_{t=1}^Tt^{-\gamma/2}\right) = O(1). \qedhere
        \end{align*}
\end{proof}

\section{Proofs from \Cref{sec:lower-bound}}
\label{sec:proofs-lower-bound}

We restate and prove \Cref{thm:lower-bound}.

\lowerBound*
\paragraph{Setup.} Fix $\mcX$ to be the interval $[-D/2, D/2] \subseteq \R$. Let $0=s_0 \le s_1 \le \dots \le s_T$ be the given sequence of cumulative curvatures. For any $t \in [T]$, let
\begin{align*}
    a_t := s_t - s_{t-1} \ge 0.
\end{align*}
The lower bound of \Cref{thm:lower-bound} holds for any sequence satisfying that $a_t \le L/2D$ for all $t$. As will become clearer, this condition is used solely to ensure that the functions used in the lower bound are all $L$-Lipschitz on $\mcX$.

Recall that 
\begin{align}
    \mfR_T(s_{1:T}) := \inf_{\mcA} \sup_{\substack{f_1,\ldots,f_T\text{ } L\text{-Lipschitz}\text{ on }\mcX \\ S_t=s_t\text{ for all }t}}
  \E_{\mcA} \left[\regret_T(f_{1:T};x_{1:T})\right].
\end{align}
and we seek to show that
\begin{align}
    \label{eqn:lb-want-to-show}
     \mfR_T(s_{1:T}) &\ge c \cdot \inf_{\rho \ge 0} \left[ L^2\sum_{t=1}^{T}\min \left\{\frac{D}{\rho}, \frac{1}{s_t}\right\} + D\rho\right].
\end{align}
It will be convenient to instead lower bound $\mfR_T(s_{1:T})$ by the quantity:
\begin{align}
    \label{eqn:regularized-form}
    \inf_{\gamma\ge 0}
    \left[
        D^2\gamma
        +
        L^2\sum_{t=1}^T\frac{1}{s_t+\gamma}
    \right].
\end{align}
Indeed, this will be sufficient, due to the following claim:
\begin{claim}
    \label{claim:regularized-clipped-equivalence}
    For every $s\ge0$ and $\gamma \ge 0$,
    \[
        \frac12\min\left\{\frac1s,\frac1\gamma\right\}
        \le
        \frac1{s+\gamma}
        \le
        \min\left\{\frac1s,\frac1\gamma\right\},
    \]
    with the convention $1/0=\infty$.
\end{claim}
\begin{proof}
    The upper bound follows because $s+\gamma\ge s$ and $s+\gamma\ge\gamma$. The lower bound follows because $s+\gamma\le 2\max\{s,\gamma\}$.
\end{proof}
Hence, if we lower-bound $\mfR_T(s_{1:T})$ in terms of the quantity in \eqref{eqn:regularized-form}, the desired lower bound in \eqref{eqn:lb-want-to-show} will follow by the change of variable $\gamma=\rho/D$.

We thus proceed to lower-bound $\mfR_T(s_{1:T})$ by $\inf_{\gamma\ge 0}
\left\{
        D^2\gamma
        +
        L^2\sum_{t=1}^T\frac{1}{s_t+\gamma}
\right\}$, and divide the proof into 8 steps:

\subsection*{Step 1: Defining the hard distribution for a fixed $\gamma$}

Fix an arbitrary deterministic learner. Also, fix $\gamma \ge 0$ which satisfies:
\begin{enumerate}
    \item \label{item:all-denominators-positive} $s_t + \gamma > 0$ for every $t$.
    \item \label{item:variance-bound} 
    $H_2(\gamma)
    :=
    \sum_{t=1}^T \frac{1}{(s_t+\gamma)^2} \le \frac{D^2}{L^2}$
\end{enumerate}
We will later choose an optimal value of $\gamma$ which satisfies these conditions.

Let $\eps_1,\dots,\eps_T$ be independent Rademacher random variables in $\{-1, +1\}$. Set 
\begin{align*}
    g := \frac{L}{8}.
\end{align*}

The functions $f_1,\dots,f_T$ will be a sequence of carefully chosen quadratics, where each $f_t$ has curvature $a_t$ (and since $a_t \ge 0$, each quadratic is convex). The centers and linear tilts $\ell_t$ of each $f_t$ are specified by the following process. 

Initialize $z_0=0$, and declare the process to be active. For each round $t=1,2,\dots,T$:
\begin{enumerate}
    \item If the process is active before round $t$, set $\ell_t=g\eps_t$; if it has been declared inactive,
    set $\ell_t=0$.
    \item Set
    \begin{equation}\label{eq:hard-loss}
        f_t(x)
        :=
        \frac{a_t}{2}(x-z_{t-1})^2
        -
        \ell_t(x-z_{t-1}).
    \end{equation}
    \item Update
    \begin{equation}\label{eq:center-update}
        z_t=z_{t-1}+\frac{\ell_t}{s_t+\gamma}.
    \end{equation}
    \item If the process was active before round $t$ and $|z_t|>D/4$, declare the process inactive.
\end{enumerate}
Observe that once the process above becomes inactive, all future linear tilts are zero and the center stays constant. Note also that all the functions $f_1,\dots,f_T$ are determined as a function of $\gamma$ and the sequences $s_1,\dots,s_T$ and $\eps_1,\dots,\eps_T$.

\subsection*{Step 2: Feasibility of centers and functions}
We first argue that all the centers $z_0,\dots,z_T$ belong to $\mcX=[-D/2,D/2]$. For this, note that assuming the bound on $H_2(\gamma)$ given in \Cref{item:variance-bound}, the magnitude of each update satisfies
\[
    \left|\frac{\ell_t}{s_t+\gamma}\right|
    \le
    g\sqrt{H_2(\gamma)}
    \le
    \frac{L}{8}\cdot\frac{D}{L}
    =
    \frac{D}{8}.
\]
Now observe that $z_0=0$, and for any $t$ where the process is active before round \(t\), we have \(|z_{t-1}|\le D/4\). Furthermore, if at round $t$, $|z_t|$ becomes larger than $D/4$ causing the process to become inactive, we have 
\[
    |z_t| \le |z_{t-1}| + D/4 \le
    D/4 + D/8
    =
    3D/8
    <
    D/2.
\]
Thereafter, the center remains fixed, and hence \(z_t\in \mcX\) for all $t \in \{0,1,\dots,T\}$.

We now argue that all the functions are feasible. Observe that even if the process has been declared inactive, the quadratic terms with coefficients $a_t$ are still
played. Thus, for every round,
\[
    f_t''(x)=a_t\ge0,
\]
so $f_t$ is convex. The cumulative loss $F_t=\sum_{i=1}^t f_i$ satisfies
\[
    F_t''(x)=\sum_{i=1}^t a_i=s_t,
\]
so the cumulative curvature is exactly the prescribed value $s_t$.

Since all the centers lie in $\mcX$, for every $x\in \mcX$,
\[
    |f_t'(x)|
    =
    |a_t(x-z_{t-1})-\ell_t|
    \le
    a_tD+|\ell_t|.
\]
Since $a_t \le L/2D$ and $|\ell_t| \le g=L/8$,
\begin{equation}\label{eq:lipschitz-bound}
    |f_t'(x)|
    \le
    \frac{L}{2}+\frac{L}{8}
    <L.
\end{equation}
Thus, every $f_t$ is convex,
$L$-Lipschitz on $\mcX$, and has the required cumulative curvature schedule.

\subsection*{Step 3: The learner's expected cumulative loss is nonnegative}

For any $t$, observe that conditioned on $\eps_1,\dots,\eps_{t-1}$, the learner's move $x_t$ is determined (since it is deterministic). Furthermore, the center $z_{t-1}$, as well as whether the process remains active/has been declared inactive before round $t$ is also determined. We consider the two cases separately.

\paragraph{Case 1: The process is active before round $t$.}
Then $\ell_t = g\eps_t$, where $\eps_t$ is independent of $\eps_1,\dots,\eps_{t-1}$, and has mean zero.
Hence,
\[
    \E[\ell_t \mid \eps_1,\dots,\eps_{t-1}] = 0.
\]
By definition of $f_t$ \eqref{eq:hard-loss},
\[
    \E[f_t(x_t)\mid \eps_1,\dots,\eps_{t-1}]
    =
    \frac{a_t}{2}(x_t - z_{t-1})^2
    \ge 0.
\]

\paragraph{Case 2: The process has become inactive before round \(t\).}
Then \(\ell_t = 0\), so
\[
    f_t(x_t)
    =
    \frac{a_t}{2}(x_t - z_{t-1})^2
    \ge 0,
\]
and therefore
\[
    \E[f_t(x_t)\mid \eps_1,\dots,\eps_{t-1}] \ge 0.
\]
In both cases,
\[
    \E[f_t(x_t)\mid \eps_1,\dots,\eps_{t-1}] \ge 0.
\]
Thus, we get that
\begin{align}
    \E\left[\sum_{t=1}^T f_t(x_t)\right] &= \sum_{t=1}^T \E[f_t(x_t)]
    = \sum_{t=1}^T \E_{\eps_1,\dots,\eps_{t-1}} \E[f_t(x_t)\mid \eps_1,\dots,\eps_{t-1}] \ge 0. \label{eq:learner-nonnegative}
\end{align}

\subsection*{Step 4: A regularized cumulative loss that has an exact quadratic invariant}

Define the fictitiously regularized cumulative loss
\[
    G_t(x):=F_t(x)+\frac{\gamma}{2}x^2.
\]
We claim that, for all $t \ge 0$,
\begin{equation}\label{eq:quadratic-invariant}
    G_t(x)
    =
    G_t(z_t)
    +
    \frac{s_t+\gamma}{2}(x-z_t)^2,
\end{equation}
and for $t \ge 1$,
\begin{equation}\label{eq:min-drop}
    G_t(z_t)
    =
    G_{t-1}(z_{t-1})
    -
    \frac{\ell_t^2}{2(s_t+\gamma)}.
\end{equation}
To see this, observe that for $t=0$, $G_0(x)=\frac{\gamma}{2}x^2$, and since $z_0=0$ and $s_0=0$, \eqref{eq:quadratic-invariant} holds. Now, let $t \ge 1$, and suppose that \eqref{eq:quadratic-invariant} holds for $t-1$. Then, we have that
\begin{align*}
    G_t(x)
    &=G_{t-1}(x)+f_t(x) \\
    &=G_{t-1}(z_{t-1})
      +\frac{s_{t-1}+\gamma}{2}(x-z_{t-1})^2
      +\frac{a_t}{2}(x-z_{t-1})^2
      -\ell_t(x-z_{t-1}) \tag{by definition of $f_t$}\\
    &=G_{t-1}(z_{t-1})
      +\frac{s_t+\gamma}{2}(x-z_{t-1})^2
      -\ell_t(x-z_{t-1}). \tag{using that $a_t = s_t -s_{t-1}$} \\
    &= G_{t-1}(z_{t-1})
      +\frac{s_t+\gamma}{2}\left[(x-z_{t-1}-\frac{\ell_t}{s_t+\gamma}\right]^2-\frac{\ell_t^2}{2(s_t+\gamma)}. \tag{completing the square}
\end{align*}
Substituting $x=z_t$, and recalling the center update from \eqref{eq:center-update} gives \eqref{eq:min-drop}. Thereafter, substituting the expression for $G_t(z_t)$ from \eqref{eq:min-drop} back in the above gives \eqref{eq:quadratic-invariant}.

\subsection*{Step 5: Telescoping the regularized cumulative loss}

Telescoping \eqref{eq:min-drop} gives
\begin{equation}\label{eq:telescoped-drop}
    G_T(z_T)
    =
    -\sum_{t=1}^T\frac{\ell_t^2}{2(s_t+\gamma)}.
\end{equation}
Since $z_T\in \mcX$, and using that $F_T(x)=G_T(x)-(\gamma/2)x^2\le G_T(x)$,
\[
    \inf_{x\in \mcX}F_T(x)
    \le
    F_T(z_T)
    \le
    G_T(z_T).
\]
Therefore,
\begin{equation}\label{eq:offline-gain}
    -\inf_{x\in \mcX}F_T(x)
    \ge - G_{T}(z_T) = 
    \sum_{t=1}^T\frac{\ell_t^2}{2(s_t+\gamma)}.
\end{equation}
Taking expectation, and combining \eqref{eq:learner-nonnegative}, \eqref{eq:offline-gain}, we get
\begin{equation}\label{eq:regret-gain-fixed-gamma}
    \E\left[\regret_T(f_{1:T};x_{1:T})\right]
    \ge
    \E\left[\sum_{t=1}^T\frac{\ell_t^2}{2(s_t+\gamma)}\right].
\end{equation}
Note that the sole randomness in the expectations above is from $\eps_1,\dots,\eps_T$. %

\subsection*{Step 6: Process remains active throughout with constant probability}

Observe that the centers $z_0, z_1,\dots,z_T$ follow a random process, and if the process is active before round $t$,
\[
    z_t
    =
    \sum_{i=1}^t \frac{\ell_i}{s_i+\gamma}.
\]
Consider a coupling $(z_t,m_t)$, where $m_t=z_t$ up until the time the process becomes inactive, but thereafter, $m_t$ continues to evolve as above even after $z_t$ has become constant. That is, for every $t$,
\begin{align*}
    m_t
    =
    \sum_{i=1}^t \frac{\ell_i}{s_i+\gamma}.
\end{align*}
Note that each term in the sum above has mean $0$. Furthermore, the variance of $m_T$ is
\begin{align*}
    \Var[m_T] = \E[m_T^2] &= g^2\sum_{t=1}^T\frac{1}{(s_t+\gamma)^2}.
\end{align*}
We are now in a position to invoke Kolmogorov's maximal inequality \citep[Theorem 22.4]{billingsley2017probability}.

\begin{proposition}[Kolmogorov's maximal inequality]
    \label{proposition:kolmogorov}
    Let $m_t=\sum_{i=1}^t x_i$, where $x_1,\ldots,x_T$ are independent mean-zero random variables with finite variance. Then, for every $r>0$,
    \[
      \Pr\left[\max_{1\le t\le T}|m_t|\ge r\right]
      \le
      \frac{\E[m_T^2]}{r^2}.
    \]
\end{proposition}
By \Cref{proposition:kolmogorov} applied to our sequence $m_1,\dots,m_T$, we get that
\begin{equation}
    \label{eq:kolmogorov-fixed-gamma}
    \Pr\left[\max_{1\le t\le T}|m_t|\ge D/4\right]
    \le
    \frac{16g^2}{D^2}\sum_{t=1}^T\frac{1}{(s_t+\gamma)^2} \le \frac{1}{4},
\end{equation}
where we used the assumed bound from \Cref{item:variance-bound} on $H_2(\gamma)$, and recalled that $g=L/8$. Thus, with probability at least $3/4$, the sequence $m_t$ never exceeds $D/4$ in absolute value. But on this event, $m_t$ coincides with $z_t$. We therefore conclude that with probability at least $3/4$, the process stays active throughout, i.e., every $\ell_t = g \eps_t$; let $A$ denote this event. Recalling the lower bound on regret from \eqref{eq:regret-gain-fixed-gamma}, we get
\begin{align}
    \E\left[\regret_T(f_{1:T};x_{1:T})\right]
    &\ge
    \E\left[\sum_{t=1}^T\frac{\ell_t^2}{2(s_t+\gamma)}\right] \nonumber \\
    &= \Pr[A] \cdot \E\left[\sum_{t=1}^T\frac{\ell_t^2}{2(s_t+\gamma)} \Bigm| A\right] + \underbrace{\Pr[\neg A] \cdot \E\left[\sum_{t=1}^T\frac{\ell_t^2}{2(s_t+\gamma)} \Bigm| \neg A\right]}_{\ge 0} \nonumber \\
    &\ge \frac{3L^2}{512} \sum_{t=1}^T\frac{1}{s_t + \gamma}. \label{eqn:regret-lb-regularized-terms}
\end{align}

\subsection*{Step 7: Choosing $\gamma$}

We now choose $\gamma$ appropriately, so that both \Cref{item:all-denominators-positive,item:variance-bound} hold. Let
\begin{align*}
        \Psi(\gamma) := D^2\gamma
        +
        L^2\sum_{t=1}^T\frac{1}{s_t+\gamma}, \qquad \gamma \ge 0.
\end{align*}
We claim that $\inf_{\gamma \ge 0} \Psi(\gamma)$ is attained at some $\gamma_\star \ge 0$, and $\gamma_\star$ satisfies both the required conditions (hence, we will choose $\gamma=\gamma_\star$). To see this, note first that 
\begin{align*}
    \Psi(2L/D) &= 2LD + L^2\sum_{t=1}^T\frac{1}{s_t+2L/D} < \infty,
\end{align*}
so let $B := \Psi(2L/D) < \infty$. Note that $B \ge 2LD$. Then, for any $\gamma > \frac{B+1}{D^2}$,
\begin{align*}
    \Psi(\gamma) &\ge D^2\gamma > D^2\left(\frac{B+1}{D^2}\right) = B+1 > B.
\end{align*}
Thus, if the infimum $\inf_{\gamma \ge 0} \Psi(\gamma)$ is attained at all, it is attained in $[0,\frac{B+1}{D^2}]$, which includes $2L/D$.

Now note that if every $s_t > 0$, then $\Psi(\gamma)$ is continuous on $[0,\frac{B+1}{D^2}]$. However, if any $s_t=0$, then the term $\frac{1}{s_t+\gamma}$ blows up at $\gamma=0$. Nevertheless, note in this case that for any $\gamma > 0$,
\begin{align*}
    \Psi(\gamma) \ge \frac{L^2}{s_t+\gamma} = \frac{L^2}{\gamma}.
\end{align*}
So, if we set $\delta := \frac{L^2}{B+1}$, then for any $0 < \gamma < \delta$,
\begin{align*}
    \Psi(\gamma) \ge \frac{L^2}{\gamma} > \frac{L^2}{\delta} = B+1.
\end{align*}
Furthermore, $\delta = \frac{L^2}{B+1} \le \frac{L^2}{2LD+1}\le \frac{L^2}{2LD} < 2L/D$. Thus, in this case, if the infimum $\inf_{\gamma \ge 0} \Psi(\gamma)$ is attained at all, it is attained in $[\delta,\frac{B+1}{D^2}]$. Furthermore, $\Psi$ is continuous on this interval.

So, define the closed and bounded interval $K \subseteq [0, \infty)$ as
\begin{align}
    K = \begin{cases}
        \left[0, \frac{B+1}{D^2}\right] & \text{if $s_t > 0$ for all $t$}, \\
        \left[\delta, \frac{B+1}{D^2}\right] & \text{otherwise}.
    \end{cases}
\end{align}
Then, it always holds that
\begin{itemize}
    \item $2L/D \in K$.
    \item For any $\gamma \ge 0$ such that $\gamma \notin K$, $\Psi(\gamma) > \Psi(2L/D)$.
    \item $\Psi$ is continuous on $K$.
\end{itemize}
Then, by the extreme value theorem, there exists $\gamma_\star \in K$ such that
\begin{align*}
    \Psi(\gamma_\star) = \min_{\gamma \in K} \Psi(\gamma) \le \Psi(2L/D).
\end{align*}
Together with the properties of $K$ above, we have shown that $\gamma_\star$ attains $\inf_{\gamma \ge 0} \Psi(\gamma)$.

We now check \Cref{item:all-denominators-positive,item:variance-bound} for $\gamma=\gamma_\star$. Note that if $s_t > 0$ for all $t$, it immediately holds that $s_t +\gamma_\star > 0$ for all $t$. On the hand, if some $s_t=0$, then $\gamma_\star > 0$, and we again get that $s_t +\gamma_\star > 0$ for all $t$. Thus, \Cref{item:all-denominators-positive} holds.

Now, for $\gamma > 0$, we have that
\begin{align*}
    \Psi'(\gamma) = D^2 -L^2\sum_{t=1}^T\frac{1}{(s_t+\gamma)^2} = D^2-L^2H_2(\gamma).
\end{align*}
If $\gamma_\star > 0$, by the first-order optimality condition, we have that
\begin{align}
    \label{eqn:first-order-optimality}
    H_2(\gamma_\star) = \frac{D^2}{L^2}.
\end{align}
On the other hand, if $\gamma_\star = 0$, the right derivative of $\Psi$ at $\gamma=0$ must be non-negative, which gives
\begin{align*}
    H_2(\gamma_\star) \le \frac{D^2}{L^2}
\end{align*}
Thus, we have shown that \Cref{item:variance-bound} holds.

\subsection*{Step 8: Putting things together, and concluding}

Recalling the regret lower bound \eqref{eqn:regret-lb-regularized-terms}, and plugging in $\gamma_\star$, we have that
\begin{align*}
    \E\left[\regret_T(f_{1:T};x_{1:T})\right] &\ge \frac{3L^2}{512} \sum_{t=1}^T\frac{1}{s_t + \gamma_\star}.
\end{align*}
We now wish to further lower bound the right-hand side above with $\Psi(\gamma_\star)$. For this, suppose that $\gamma_\star=0$. In this case,
\begin{align*}
    \Psi(\gamma_\star) = L^2\sum_{t=1}^T \frac{1}{s_t+\gamma_\star}. %
\end{align*}
On the other hand, if $\gamma_\star > 0$, we have by the first-order optimality condition \eqref{eqn:first-order-optimality} that
\begin{align*}
    &D^2 = L^2\sum_{t=1}^T\frac{1}{(s_t+\gamma_\star)^2} \\
    \quad\implies& D^2\gamma_\star = L^2\sum_{t=1}^T\frac{\gamma_\star}{(s_t+\gamma_\star)^2} = L^2\sum_{t=1}^T\frac{1}{s_t+\gamma_\star}\cdot \frac{\gamma_\star}{s_t+\gamma_\star} \le L^2\sum_{t=1}^T\frac{1}{s_t+\gamma_\star},
\end{align*}
and hence,
\begin{align*}
    \Psi(\gamma_\star) = D^2\gamma_\star + L^2\sum_{t=1}^T \frac{1}{s_t+\gamma_\star} \le 2L^2\sum_{t=1}^T \frac{1}{s_t+\gamma_\star}.
\end{align*}
In either case, we are guaranteed that $\Psi(\gamma_\star) \le 2L^2\sum_{t=1}^T \frac{1}{s_t+\gamma_\star}$; substituting this in our regret bound above, we get
\begin{align}
    \label{eqn:deterministic-alg-lower-bound}
    \E\left[\regret_T(f_{1:T};x_{1:T})\right] &\ge \frac{3}{1024}\Psi(\gamma_\star) = \frac{3}{1024}\cdot \inf_{\gamma \ge 0} \left[
        D^2\gamma
        +
        L^2\sum_{t=1}^T\frac{1}{s_t+\gamma}
    \right].
\end{align}

The bound above holds for every deterministic learning algorithm, over the randomness of $\eps_1,\dots,\eps_T$. This implies a corresponding lower bound for any randomized algorithm, by the standard Yao argument \citep{yao1983lower}. Concretely, let $\mcA$ be any randomized algorithm. Conditioned on any fixed realization $r$ of its randomness, the algorithm $\mcA_r$ is deterministic. Thus, we can invoke the lower bound in \eqref{eqn:deterministic-alg-lower-bound} for $\mcA_r$, to get
\begin{align*}
    \E_{\eps_{1:T}}\left[\regret_{\mcA_r, T}(f_{1:T};x_{1:T})\right] &\ge \frac{3}{1024}\cdot \inf_{\gamma \ge 0} \left[
        D^2\gamma
        +
        L^2\sum_{t=1}^T\frac{1}{s_t+\gamma}
    \right].
\end{align*}
Averaging over the randomness $r$ then gives that
\begin{align*}
    \E_r\E_{\eps_{1:T}}\left[\regret_{\mcA_r, T}(f_{1:T};x_{1:T})\right] = &\E_{\eps_{1:T}} \E_r\left[\regret_{\mcA_r, T}(f_{1:T};x_{1:T})\right] \ge \frac{3}{1024}\cdot \inf_{\gamma \ge 0} \left[
        D^2\gamma
        +
        L^2\sum_{t=1}^T\frac{1}{s_t+\gamma}
    \right].
\end{align*}
In particular, this implies the existence of a fixed sequence $\eps_1,\dots,\eps_T$, and corresponding functions $f_1,\dots,f_T$, such that
\begin{align*}
    \E_r\left[\regret_{\mcA_r, T}(f_{1:T};x_{1:T})\right] \ge \frac{3}{1024}\cdot \inf_{\gamma \ge 0} \left[
        D^2\gamma
        +
        L^2\sum_{t=1}^T\frac{1}{s_t+\gamma}
    \right].
\end{align*}
Since the lower bound above holds for every randomized learning algorithm $\mcA$, we get
\begin{align*}
    \mfR_T(s_{1:T}) &\ge \frac{3}{1024}\cdot \inf_{\gamma \ge 0} \left[
        D^2\gamma
        +
        L^2\sum_{t=1}^T\frac{1}{s_t+\gamma}
    \right] \\
    &\ge c \cdot \inf_{\rho \ge 0} \left[ L^2\sum_{t=1}^{T}\min \left\{\frac{D}{\rho}, \frac{1}{s_t}\right\} + D\rho\right]. \tag{using \Cref{claim:regularized-clipped-equivalence}}
\end{align*}
This concludes the proof of \Cref{thm:lower-bound}. \qed

\subsection{Optimality of regime-specific rates}
\label{sec:regime-specific-lower-bounds}

Here, we argue that the rates obtained by our algorithm in the different curvature in \Cref{corollary:precise-regimes} are optimal.

Recall from \Cref{thm:lower-bound} that the minimax regret is lower bounded (upto constant factors) by $\opt^0_{T}(0)$, where
\begin{align*}
    \opt^0_{T}(0) \ge c \cdot \underbrace{\inf_{\rho \ge 0} \left[ \sum_{t=1}^{T}\min \left\{\frac{1}{\rho}, \frac{1}{S_t}\right\} + \rho\right].}_{:= \Psi(\rho)}
\end{align*}
Here, the constant $c$ hides a dependence on $d, D, L$. It thus suffices to calculate $\Psi(\rho)$ in the regimes where $S_t=t^\gamma$ (for $\gamma \ge 1/2$), and show a matching characterization as that in \Cref{corollary:precise-regimes}.

Consider 
\begin{align*}
    \Psi(\rho)= \inf_{\rho \ge 0} \left[ \sum_{t=1}^{T}\min \left\{\frac{1}{\rho}, \frac{1}{S_t}\right\} + \rho\right],
\end{align*}
for $S_t = t^\gamma$.

Towards this, we first further characterize the precise form of $\Psi(\rho)$. For $0 \le t < T$, consider the intervals $I_t = [t^\gamma, (t+1)^\gamma)$, and let $I_T=[T^\gamma, \infty)$. For any $\rho \in I_t$, observe that $\Psi(\rho)$ evaluates to
\begin{align*}
    \frac{t}{\rho} + \rho + \sum_{t < i \le T} i^{-\gamma}.
\end{align*}
Note that the last term is constant for every $\rho \in I_t$. Thus, the ``unconstrained'' minimizer for this interval balances the first two terms, giving 
\begin{align*}
    \hat{\rho}_t = \sqrt{t}.
\end{align*}
This is a valid value for $\hat{\rho}_t$, provided $\sqrt{t} \in I_t$. However, if $\sqrt{t} \notin I_t$, we simply set $\hat{\rho}_t$ to be either $t^\gamma$ or $(t+1)^{\gamma}$, based on whichever of the two evaluates to a smaller $\Psi(\rho)$. Thus, we get that
\begin{align}
    \label{eqn:precise-form-opt-0}
    \inf_{\rho \ge 0} \Psi(\rho) = \min_{0 \le t \le T} \left(\frac{t}{\hat{\rho}_t}+\hat{\rho}_t + \sum_{t < i \le T} i^{-\gamma}\right).
\end{align}
When $\gamma < 1/2$, it may be verified that the minimum above gets achieved at $t=T$ (for large enough $T$), whereby $\hat{\rho}_T=\sqrt{T}$, giving that the quantity in \eqref{eqn:precise-form-opt-0} is at least $T/\sqrt{T}+\sqrt{T}=\Omega(\sqrt{T})$. When $\gamma \ge 1/2$, the minimum in \eqref{eqn:precise-form-opt-0} gets achieved at $t=0$ as $\hat{\rho}_0 \to 0$, whereby the quantity in \eqref{eqn:precise-form-opt-0} is at least $\sum_{i=1}^Ti^{-\gamma}$, which gives the claimed lower bounds on the rates in each regime of $\gamma$.

\subsection{Relating $\opt^0_{T}(r)$ and $\opt^0_{T}(0)$}
\label{sec:relating-opt-0-r-to-opt-0-0}

\begin{claim}
    \label{claim:relating-opt-0-r-to-opt-0-0}
    Recall the definition of $\opt^0_{T}(r)$:
    \begin{align*}
        \opt^0_{T}(r) = \inf_{\rho \ge r} \left[ \sum_{t=1}^{T}\min \left\{\frac{125L^2d^2D}{\rho}, \frac{2L^2}{S_t}\right\} + dD\rho\right].
    \end{align*}
    For any $r \ge 0$, it holds that
    \begin{align*}
        \opt^0_{T}(r) &\le \opt^0_{T}(0) + dDr.
    \end{align*}
\end{claim}
\begin{proof}
    For any $\rho \ge 0$, let $\tilde{\rho}=\max(r, \rho)$. Then observe that for any $t$,
    \begin{align*}
        \min \left\{\frac{125L^2d^2D}{\tilde{\rho}}, \frac{2L^2}{S_t}\right\} &\le \min \left\{\frac{125L^2d^2D}{\rho}, \frac{2L^2}{S_t}\right\}.
    \end{align*}
    Furthermore,
    \begin{align*}
        dD\tilde{\rho} \le dD\rho + dDr.
    \end{align*}
    Therefore,
    \begin{align*}
        \sum_{t=1}^{T}\min \left\{\frac{125L^2d^2D}{\tilde{\rho}}, \frac{2L^2}{S_t}\right\} + dD\tilde\rho &\le  \sum_{t=1}^{T}\min \left\{\frac{125L^2d^2D}{\rho}, \frac{2L^2}{S_t}\right\} + dD\rho + dDr.
    \end{align*}
    Infimizing over $\rho \ge 0$ gives,
    \begin{align*}
        &\inf_{\rho \ge 0}\left[\sum_{t=1}^{T}\min \left\{\frac{125L^2d^2D}{\tilde{\rho}}, \frac{2L^2}{S_t}\right\} + dD\tilde\rho\right] \le  \inf_{\rho \ge 0}\left[\sum_{t=1}^{T}\min \left\{\frac{125L^2d^2D}{\rho}, \frac{2L^2}{S_t}\right\} + dD\rho + dDr\right] \\
        \implies & \inf_{\rho \ge r}\left[\sum_{t=1}^{T}\min \left\{\frac{125L^2d^2D}{\rho}, \frac{2L^2}{S_t}\right\} + dD\rho\right] \le  \inf_{\rho \ge 0}\left[\sum_{t=1}^{T}\min \left\{\frac{125L^2d^2D}{\rho}, \frac{2L^2}{S_t}\right\} + dD\rho\right] + dDr \\
        \implies & \opt^0_{T}(r) \le \opt^0_{T}(0) + dDr. \qedhere
    \end{align*}
\end{proof}

\end{document}